\documentclass{article} % For LaTeX2e
\usepackage{iclr2027_conference,times}

\usepackage[utf8]{inputenc}
\usepackage[T1]{fontenc}
\usepackage{hyperref}
\usepackage{url}
\usepackage{booktabs}
\usepackage{nicefrac}
\usepackage{microtype}
\usepackage{xcolor}
\usepackage{graphicx}
\usepackage{placeins}

\usepackage{amsmath,amssymb,amsfonts,amsthm,mathtools}
\usepackage{mathrsfs}
\usepackage{bbm}
\usepackage{enumitem}
\usepackage{comment}

\usepackage{tcolorbox}
\tcbuselibrary{skins,breakable}

\newcommand{\cA}{\mathcal{A}}
\newcommand{\cD}{\mathcal{D}}
\newcommand{\cF}{\mathcal{F}}

\newcommand{\cL}{\mathcal{L}}

\newcommand{\cX}{\mathcal{X}}

\newcommand{\EE}{\mathbb{E}}
\newcommand{\PP}{\mathbb{P}}
\newcommand{\R}{\mathbb{R}}

\DeclareMathOperator{\supp}{supp}

\DeclareMathOperator{\KL}{KL}
\DeclareMathOperator{\Var}{Var}

\DeclareMathOperator{\Regret}{Regret}

\DeclareMathOperator*{\argmax}{arg\,max}
\DeclareMathOperator*{\argmin}{arg\,min}

\newcommand{\eps}{\varepsilon}

\theoremstyle{plain}
\newtheorem{theorem}{Theorem}
\newtheorem{lemma}[theorem]{Lemma}
\newtheorem{corollary}[theorem]{Corollary}
\newtheorem{proposition}[theorem]{Proposition}

\theoremstyle{definition}

\theoremstyle{remark}

\tcbset{
  paperresult/.style={
    enhanced jigsaw,
    breakable,
    boxrule=0pt,
    frame hidden,
    sharp corners,
    before skip=8pt,
    after skip=8pt,
    left=2pt,
    right=2pt,
    top=2pt,
    bottom=2pt,
    pad at break*=2mm
  }
}
\tcolorboxenvironment{definition}{paperresult,colback=black!5}
\tcolorboxenvironment{theorem}{paperresult,colback=blue!5}
\tcolorboxenvironment{proposition}{paperresult,colback=blue!5}
\tcolorboxenvironment{corollary}{paperresult,colback=blue!5}
\tcolorboxenvironment{lemma}{paperresult,colback=teal!5}

\title{Demystifying the Unreasonable Effectiveness of Greedy Alignment Methods}

\author{%
  Enoch Hyunwook Kang \\
  Foster School of Business \\
  University of Washington \\
  \texttt{ehwkang@uw.edu}
}

\iclrfinalcopy % Uncomment only for the camera-ready version.

\begin{document}

\maketitle
\lhead{Preprint}
\begin{abstract}
Greedy AI alignment methods based on Bradley-Terry reward-model fitting and KL-regularized Alignment are remarkably effective in both offline and online preference-learning pipelines, yet existing $O(\log T)$ online and $O(1/\epsilon)$ offline KL-regularized regret rate guarantees can seem pessimistic relative to their empirical performance. We argue that this mismatch reflects a distinction between two learning targets: KL-regularized regret measures how fast we can recover the KL-regularized aligned policy, whereas we often only ask how fast we can learn the reward function that induces the correct best response. To isolate this reward-learning component, we study the traditional temperature-zero regret criterion, which evaluates only the top-ranked response at inference time. Under realizability, compactness, and a uniform best-response margin over the fitted reward class, we prove that exact reference-logged offline RLHF achieves exponentially small expected temperature-zero regret, and therefore reaches expected regret at most \(\epsilon\) with \(O(\log(1/\epsilon))\) offline comparisons. We also prove that exact greedy online RLHF achieves bounded \(O(1)\) cumulative temperature-zero regret; the same guarantee extends to its equivalent pairwise online DPO formulation. By separating selector identification from full soft-policy recovery, our results provide a sharper theoretical explanation for the empirical efficiency of greedy alignment.
\end{abstract}

\section{Introduction}

Preference-based alignment \citep{christiano2017deep, ouyang2022instructgpt} trains AI systems to favor responses that
people prefer. A platform, implicitly or explicitly, fits a \emph{reward model} from comparisons
between responses and then shifts the AI's response distribution
toward higher-scoring outputs. Offline methods learn from a fixed
comparison log; online methods repeatedly generate responses, collect
feedback, and update the model. Both commonly use Kullback--Leibler
(KL) regularization, which balances higher estimated rewards against
departures from a reference model.

%A key distinction is that \emph{an aligned model can continue to generate inferior responses even after it has learned which response is best}. KL-regularized alignment produces a randomized policy---a distribution over responses---rather than always choosing the highest-scoring response. At finite temperature, this policy retains probability on inferior responses supported by the reference model, even when the true reward function is known. Continued suboptimal generation can therefore reflect deployment randomization rather than unfinished preference learning.

Recent theory establishes offline and online guarantees for such
KL-regularized alignment
\citep{zhao2024sharp,zhao2025online,zhao2025offline}.
In particular, \citet{wu2025greedy} show that greedy online alignment
achieves $O(\log T)$ cumulative KL-regularized regret without explicit
optimism or pessimism. Here, \emph{greedy} means optimizing the
current estimated reward subject to KL regularization, not selecting
responses deterministically; KL-regularized regret measures the gap between the learner’s randomized policy and the optimal KL-regularized policy, which deliberately retains probability on lower-reward responses in order to remain close to the reference policy.

A central point of this paper is that such logarithmic KL-regularized regret guarantees need not describe how quickly the underlying reward model is being learned. KL-regularized regret asks how quickly the learner recovers an entire randomized, regularized policy. But once alignment is complete, many downstream inference-time applications do not require recovery of the entire randomized policy; they only require that the learned reward model correctly identifies which response is best \citep{nakano2021webgpt,snell2024scaling}. Existing KL-regularized regret guarantees therefore leave unresolved how quickly greedy alignment learns this simpler, decision-relevant object: the identity of the best response.
This distinction leaves an open question:
\emph{How large data/how long does greedy alignment need to learn a reward model accurate enough to identify which response is best?} 

%In this paper, we show that \emph{greedy online alignment can complete best-response learning after an almost-surely finite random time, while its deployed KL-regularized policy continues to randomize}. Under minor regularization conditions, the probability of an incorrect best-response rule is exponentially small. Consequently, after a finite random learning period, every subsequent fitted reward identifies the optimal response for almost every context. The numerical reward estimates may continue to change, but the response they identify as best no longer does. Thus, policy-regret guarantees alone can understate the efficiency of greedy alignment at the best-response task. 

To answer this question, we propose a traditional-style regret notion we call \emph{temperature-zero regret}. Rather than evaluating the entire randomized policy induced by the learned reward, temperature-zero regret evaluates only the response that the learned reward ranks highest. It asks a simple question: if we use the learned reward model to choose its top response, how much reward do we lose relative to choosing the truly best response? By contrast, KL-regularized regret asks whether the entire finite-temperature response distribution has converged to the optimal regularized distribution. The former therefore measures best-response identification; the latter measures full soft-policy recovery.

This distinction matters when reward models are ultimately used in downstream tasks to rank or select outputs at inference time, where learning the exact probabilities assigned to all non-selected responses may be less important than correctly identifying the top-ranked response. Temperature-zero regret provides a direct measure of this decision-relevant component of learning. Methodologically, this criterion is closely related to simple regret in contextual bandits \citep{deshmukh2018simpleregret,krishnamurthy2023proportional}; our contribution is not the argmax criterion itself, but its analysis for the same greedy KL-regularized RLHF/DPO process, including adaptive online preference collection.

%Importantly, we do not change the alignment algorithm or its data-collection process. Online comparisons are still collected from the randomized KL-regularized policy, and the reward model is still fitted to accumulated preference feedback. Indeed, this randomization helps sustain informative data collection: under our bounded-reward conditions, the KL-regularized policy continues to place sufficient probability on responses supported by the reference model. The randomization that supports learning therefore need not disappear when best-response identification is complete. We change only how the resulting learning progress is evaluated.

\paragraph{Main contributions.}
Under minor realistic regularity conditions, we establish two results.
\begin{itemize}[leftmargin=*]

\item \textbf{Bounded online regret.}
Standard exact greedy reinforcement learning from human feedback
(RLHF) has bounded $O(1)$ expected cumulative temperature-zero regret
(Theorem~\ref{thm:np_bounded_regret}). The same guarantee holds for
exact pairwise online direct preference optimization
(DPO) over the reward-induced policy class
(Corollary~\ref{cor:np_bounded_regret_dpo} in
Appendix~\ref{app:greedy_alignment_details}).
This bound evaluates the induced best-response rule, not the
randomized responses sampled to collect feedback.

\item \textbf{Exponential offline accuracy.}
For exact offline RLHF trained on comparisons generated by the
reference policy, expected temperature-zero regret decreases
exponentially with sample size. Consequently,
$O(\log(1/\epsilon))$ comparisons suffice to achieve expected regret
at most $\epsilon$
(Theorem~\ref{thm:np_offline_log_accuracy}).

\end{itemize}

\begin{table}[ht!]
\centering
\caption{\textbf{Greedy alignment under two evaluation criteria.}
KL-regularized entries summarize prior guarantees
\citep{wu2025greedy}; bold entries are our expected-regret guarantees
(Theorems~\ref{thm:np_bounded_regret}
and~\ref{thm:np_offline_log_accuracy}).}
\label{tab:regret_comparison}
\small
\renewcommand{\arraystretch}{1.2}
\begin{tabular*}{\linewidth}{@{\extracolsep{\fill}}lcc@{}}
\toprule
& \shortstack{KL-regularized regret \\\citep{wu2025greedy}}
& \shortstack{\textbf{Temperature-zero regret}\\\textbf{(this work)}} \\
\midrule
Learning target
& \shortstack{Full soft-policy\\recovery}
& \shortstack{Learning\\reward model} \\
\addlinespace[3pt]
Online cumulative regret ($T$ rounds)
& $O(\log T)$ & $\boldsymbol{O(1)}$ \\
Offline one-step regret ($n$ comparisons)
& $O(1/n)$ & $\boldsymbol{O(e^{-cn})}$ \\
Offline comparisons for regret $\leq\eps$
& $O(1/\eps)$ & $\boldsymbol{O(\log(1/\eps))}$ \\
\bottomrule
\end{tabular*}
\vspace{3pt}
\begin{minipage}{\linewidth}
\end{minipage}
\end{table}

A controlled simulation illustrates the distinction: one-step
temperature-zero regret rapidly approaches zero and cumulative
regret levels off, although preference collection remains randomized.
Together, these results sharpen the interpretation of logarithmic
policy-regret guarantees:
\emph{the reward model can finish learning which response to choose
while the randomized policy is still being refined}.

Section~\ref{sec:setup} introduces the model and the offline and online
alignment procedures. Section~\ref{sec:nonpersonalized_online_alignment}
defines temperature-zero regret and presents the theoretical results,
and Section~\ref{sec:experiment} reports the simulation. Related work
appears in Appendix~\ref{sec:related}.

\section{Setup and greedy alignment}
\label{sec:setup}
\label{sec:np_greedy_alignment_procedure}

We study a common alignment pipeline: fit a reward model to preference
feedback, then use the fitted reward to reweight a fixed reference policy.\footnote{For pairwise feedback, this pipeline also admits an equivalent DPO \citep{rafailov2023dpo} formulation: under exact optimization over the reward-induced policy class, DPO is a reparameterization of greedy RLHF and yields the same policy iterates and temperature-zero selectors. See Appendix~\ref{app:greedy_alignment_details}.}
The offline and online procedures we will introduce now share this pipeline. The distinction is how preference data are collected: offline
learning uses a fixed reference-policy log, whereas online learning updates
the sampling policy each round.

Let \(\cX\) be a context space (prompts, queries, or tasks), and let
\(\cA\) be a response space, which may be large or infinite. We do not assume that \(\cA\) is continuous; for LLMs, responses are naturally discrete token sequences, albeit from an enormous output space. Contexts
are drawn independently from an unknown distribution \(d_0\). A base AI model \(\pi_0:\cX\to\Delta(\cA)\) represents the
pre-alignment model, which we use as the starting point for our alignment procedure, such as a pretrained model or a supervised fine-tuned checkpoint. %Werestrict attention to policies satisfying \(\pi(\cdot\mid x)\ll\pi_0(\cdot\mid x)\) for every \(x\): alignment reweights the reference distribution rather than assigning probability to sets of responses that have zero reference probability.

Throughout, \(\cA\) is a separable metric space equipped with its Borel
\(\sigma\)-field, and \(\pi_0\) and the sampling policies are measurable
probability kernels. Whenever a maximizing selector is used, we assume that
the maximum is attained and that a measurable selector is available.
Tie-breaking depends only on the maximizing set, so it is unchanged by
adding an action-independent constant to a reward. All such selectors take
values in the reference support. These are standing well-posedness
conventions; compactness of a reward class alone does not imply attainment
of a maximum over a noncompact action space.

Fix a slate size \(K\ge2\). Each preference observation consists of a
context \(x\), a slate \(\mathbf a=(a_1,\ldots,a_K)\), and a
preferred index \(y\in\{1,\ldots,K\}\), drawn from an unknown choice
law \(P^\star(\cdot\mid x,\mathbf a)\). The case \(K=2\)
corresponds to pairwise feedback; larger \(K\) allows preferences over
larger slates.

\subsection{Reward fitting and KL-regularized AI alignment}
\label{sec:np_exact_erm}

\paragraph{Reward fitting}
Let \(\cF\) be the learner's class of measurable reward functions
\(R:\cX\times\cA\to\R\). The learner associates each candidate
reward with the multinomial logit preference model\footnote{The general
preference extension is described in Appendix~\ref{app:gp_temperature_zero}.}
\[
P_R(y=k\mid x,\mathbf a)
=
\frac{\exp(R(x,a_k))}{\sum_{\ell=1}^K\exp(R(x,a_\ell))},
\qquad k=1,\ldots,K.
\]
For \(K=2\), this is the Bradley--Terry model. Write
\(\mathbf v_R(x,\mathbf a)=(R(x,a_1),\ldots,R(x,a_K))\) and
\(\ell(\mathbf v,y)=\log(\sum_{k=1}^K e^{v_k})-v_y\).
Given a dataset \(\cD_n=\{(x_s,\mathbf a_s,y_s)\}_{s=1}^n\),
the empirical preference loss and fitted reward are
\begin{align}
\widehat{\cL}_n(R)
&\coloneqq
\frac1n\sum_{s=1}^n
\ell\!\big(\mathbf v_R(x_s,\mathbf a_s),y_s\big), \quad
\widehat R_n
\in\argmin_{R\in\cF}\widehat{\cL}_n(R).
\label{eq:np_exact_erm}
\end{align}
The assumptions on
\(\cF\) and its relation to the true feedback law \(P^\star\)
are stated with the theoretical results in
Section~\ref{sec:nonpersonalized_online_alignment}. When the sample size or iteration index is immaterial, we suppress the subscript and write \(\widehat R\) for a generic fitted reward.

\paragraph{KL-regularized AI alignment}
We then use the fitted reward $\widehat R$ to align the AI with KL regularization.
For each context \(x\), fix a constant \(\eta>0\) and assume that the
fitted reward is bounded on the reference support,
$
\big|\widehat R(x,a)\big|\le B$ for $\pi_0(\cdot\mid x)\text{-a.e. }a,
$
for some finite \(B\). The policy-improvement step maximizes
\[
\EE_{A\sim\pi(\cdot\mid x)}[\widehat R(x,A)]
-\eta^{-1}\KL\!\big(
\pi(\cdot\mid x)\,\|\,\pi_0(\cdot\mid x)
\big).
\]
Its solution is the exponential tilt
\begin{equation}
\pi_{\widehat R}(da\mid x)
=
\frac{\exp\!\big(\eta\widehat R(x,a)\big)\,\pi_0(da\mid x)}
{\int_{\cA}\exp\!\big(\eta\widehat R(x,a')\big)\,\pi_0(da'\mid x)}.
\label{eq:np_greedy_policy_update_general}
\end{equation}
which satisfies
\[
e^{-2\eta B}
\le
\frac{d\pi_{\widehat R}(\cdot\mid x)}
     {d\pi_0(\cdot\mid x)}(a)
\le
e^{2\eta B}
\qquad
\pi_0(\cdot\mid x)\text{-a.s.},
\]
so \(\pi_{\widehat R}(\cdot\mid x)\) and
\(\pi_0(\cdot\mid x)\) are mutually absolutely continuous and therefore
have the same null sets. Thus, KL-regularized AI alignment reweights responses already
supported by the base AI, rather than assigning probability to response
regions that the base AI rules out; see
Lemma~\ref{lem:greedy_kl_tilt_lr} in
Appendix~\ref{app:proofs_np_online_alignment}.
Both \(\eta\) and the reference policy \(\pi_0\) remain fixed throughout
learning.

\subsection{Greedy Offline and online AI alignment}
\label{sec:np_offline_rlhf_procedure}

\paragraph{Offline RLHF.}
The learner receives a fixed dataset collected independently across
\(s=1,\ldots,n\) according to
\[
X_s\sim d_0,
\qquad
\mathbf A_s\sim\pi_0(\cdot\mid X_s)^{\otimes K},
\qquad
Y_s\sim P^\star(\cdot\mid X_s,\mathbf A_s).
\]
It fits \(\widehat R_n^{\mathrm{off}}\) using
\eqref{eq:np_exact_erm} and forms the aligned policy
\(\pi_{\widehat R_n^{\mathrm{off}}}\) using
\eqref{eq:np_greedy_policy_update_general}.
The fitted model does not influence the data in this fixed log.

\paragraph{Greedy online RLHF.}
Following \citet{wu2025greedy}, initialize
\(\widehat R_0\equiv0\), \(\pi_1=\pi_0\), and
\(\cD_0=\emptyset\). At round \(t\), collect an observation
according to
\[
X_t\sim d_0,
\qquad
\mathbf A_t\sim\pi_t(\cdot\mid X_t)^{\otimes K},
\qquad
Y_t\sim P^\star(\cdot\mid X_t,\mathbf A_t).
\]
Append \((X_t,\mathbf A_t,Y_t)\) to \(\cD_{t-1}\), refit
\(\widehat R_t\) on the accumulated dataset \(\cD_t\) using
\eqref{eq:np_exact_erm}, and set
\(\pi_{t+1}=\pi_{\widehat R_t}\) using
\eqref{eq:np_greedy_policy_update_general}.
Thus, the policy used to collect round-\(t\) feedback is based on the
first \(t-1\) observations, while the resulting fit is used from
round \(t+1\). Every update remains relative to the original reference
policy \(\pi_0\), not the preceding policy \(\pi_t\).

Both procedures produce a fitted reward and its associated randomized
policy. Our analysis evaluates the response ranked highest by the fitted
reward, rather than the full randomized policy. The next section
formalizes this distinction through temperature-zero regret. This
criterion changes only the evaluation target: online preference
collection continues to use the finite-temperature policy \(\pi_t\).

\section{Temperature-zero regret analysis of greedy alignment methods}
\label{sec:nonpersonalized_online_alignment}

We now formally introduce the temperature-zero regret criterion and analyze the alignment procedures from Section~\ref{sec:np_greedy_alignment_procedure}. Throughout this section, all rewards in the fitted class are \(\pi_0\)-centered, that is, \(\EE_{A\sim\pi_0(\cdot\mid x)}[R(x,A)]=0\) for \(d_0\)-almost every \(x\).
 This is without loss of generality: subtracting, for each context \(x\), an \(a\)-independent constant from a reward leaves the MNL choice probabilities, the KL-tilted policy, and the temperature-zero selector unchanged; see Lemma~\ref{lem:pi0_centering_wlog_psi} of
Appendix~\ref{app:proofs_np_online_alignment}. All envelope constants below
refer to the centered class; centering a class with envelope \(B\) gives
an envelope of at most \(2B\).

Let \(\mathcal P\) denote the class of possible truths, where each \(p\in\mathcal P\) induces a centered measurable reward \(R_p:\cX\times\cA\to\R\), and we let \(\mathcal F_{\mathcal P}\coloneqq\{R_p:p\in\mathcal P\}\). Under truth \(p\), preference feedback on a realized slate \(\mathbf a=(a_1,\dots,a_K)\in\cA^K\) follows \(P_{R_p}(y=k\mid x,\mathbf a)=\exp(R_p(x,a_k))/\sum_{\ell=1}^K\exp(R_p(x,a_\ell))\), for \(k=1,\dots,K\). For \(x\in\cX\), write \(S_x\coloneqq\supp(\pi_0(\cdot\mid x))\).

\subsection{Temperature zero regret}

For each centered reward under consideration, use the standing
measurable selector
\[
a_R(x)\in\argmax_{a\in S_x}R(x,a).
\] This tie-breaking convention is used throughout; for truth \(p\in\mathcal P\), write \(a_p(x)\coloneqq a_{R_p}(x)\). For any centered reward \(R\), define the expected \textit{one-step temperature-zero regret} under truth \(p\) by
\begin{equation}
\mathcal G_p(R)
\coloneqq
\EE_{X\sim d_0}
\Big[
R_p\big(X,a_p(X)\big)-R_p\big(X,a_R(X)\big)
\Big].
\label{eq:np_truth_centered_expected_regret}
\end{equation}

In the offline regime, the learner does not act while the fixed log is
being collected. We therefore evaluate offline RLHF by
\(\EE_p[\mathcal G_p(\widehat R_n^{\mathrm{off}})]\), or equivalently
by the sample size needed to make this expected one-step temperature-zero
regret at most \(\eps\), rather than by cumulative learning regret.

A learning rule \(\mathsf A\) produces estimates \(\widehat R_0^{\mathsf A},\widehat R_1^{\mathsf A},\widehat R_2^{\mathsf A},\dots\), where \(\widehat R_t^{\mathsf A}\) is fitted from the first \(t\) online rounds and deployed on round \(t+1\). Its expected \textit{cumulative temperature-zero regret} under truth \(p\in\mathcal P\) is
\begin{equation}
\Regret_{0,p}^{\mathsf A}(T)
\coloneqq
\sum_{t=0}^{T-1}\EE_p\big[\mathcal G_p(\widehat R_t^{\mathsf A})\big].
\label{eq:np_expected_cumulative_regret}
\end{equation}

\paragraph{cf. KL-regularized regret.} Instead of temperature-zero regret used in this paper, \citet{wu2025greedy} uses a regret criterion called KL-regularized regret.
Fix the same tilt parameter \(\eta>0\) and reference policy \(\pi_0\)
as in the alignment step. For a policy satisfying
\(\pi(\cdot\mid x)\ll\pi_0(\cdot\mid x)\), define its
true KL-regularized value by
\begin{equation}
\mathcal J_{\eta,p}(\pi)
\coloneqq
\EE_{X\sim d_0}\!\left[
\EE_{A\sim\pi(\cdot\mid X)}[R_p(X,A)]
-\frac1\eta\KL\!\big(
\pi(\cdot\mid X)\,\|\,\pi_0(\cdot\mid X)
\big)
\right].
\label{eq:np_kl_regularized_value}
\end{equation}
Under the boundedness conditions below, its maximizer is
\(\pi_{\eta,p}^{\star}\coloneqq\pi_{R_p}\), obtained from
\eqref{eq:np_greedy_policy_update_general} using the true reward.
For a bounded candidate reward \(R\), define the associated
\emph{one-step KL-regularized regret} by
\begin{equation}
\begin{aligned}
\mathcal G_{\eta,p}^{\mathrm{KL}}(R)
&\coloneqq
\mathcal J_{\eta,p}(\pi_{\eta,p}^{\star})
-\mathcal J_{\eta,p}(\pi_R)\\
&=
\frac1\eta\EE_{X\sim d_0}\!\left[
\KL\!\big(
\pi_R(\cdot\mid X)\,\|\,\pi_{\eta,p}^{\star}(\cdot\mid X)
\big)
\right].
\end{aligned}
\label{eq:np_kl_regularized_one_step_regret}
\end{equation}
Indeed, applying \eqref{eq:kl_tilt_variational_identity} with reward
\(R_p\) and policy \(\pi_R\) gives, at each context,
\[
\int R_p(x,a)\,\pi_R(da\mid x)
-\eta^{-1}\KL(\pi_R(\cdot\mid x)\|\pi_0(\cdot\mid x))
=\eta^{-1}\log Z_{R_p}(x)
-\eta^{-1}\KL(\pi_R(\cdot\mid x)\|\pi_{\eta,p}^{\star}(\cdot\mid x)).
\]
For \(\pi_R=\pi_{\eta,p}^{\star}\), the last KL term is zero.
Subtracting the two values and integrating over \(d_0\) yields
\eqref{eq:np_kl_regularized_one_step_regret}.
The corresponding expected cumulative regret is
\begin{equation}
\Regret_{\mathrm{KL},\eta,p}^{\mathsf A}(T)
\coloneqq
\sum_{t=0}^{T-1}
\EE_p\!\left[
\mathcal G_{\eta,p}^{\mathrm{KL}}(\widehat R_t^{\mathsf A})
\right],
\label{eq:np_kl_expected_cumulative_regret}
\end{equation}
while the offline criterion is
\(\EE_p[\mathcal G_{\eta,p}^{\mathrm{KL}}(\widehat R_n^{\mathrm{off}})]\).

The two criteria evaluate different objects produced by the same fit.
Temperature-zero regret vanishes once \(a_R(X)\) is optimal under
\(R_p\) almost surely, whereas KL-regularized regret vanishes only
when \(\pi_R(\cdot\mid X)=\pi_{\eta,p}^{\star}(\cdot\mid X)\)
for \(d_0\)-almost every context. Thus, a fitted reward can identify
the correct best response while still inducing an imperfect
finite-temperature distribution. The distinction is between
\emph{best-response identification} and \emph{full soft-policy recovery},
not between learning error and an unavoidable randomization cost.

\subsection{Model class}
\label{sec:np_model_class_eval}

For a function on the reference support, write
$
\|R\|_{\infty,\supp(\pi_0)}
\coloneqq \sup_{x\in\cX}\sup_{a\in S_x}|R(x,a)|.
$
We use this supportwise supremum norm, rather than an essential supremum
over actions, because the evaluation rule maximizes over all of \(S_x\).
Functions that agree on the reference support are identified when this
norm is used. Assume that the class of possible truths \(\mathcal P\) is nonempty and compact.
We impose the following regularity conditions.
\begin{enumerate}[label=(C\arabic*),leftmargin=2em]
\item \label{as:NP1} \textbf{Compact continuous reward class.}
The induced reward class \(\mathcal F_{\mathcal P}\) is compact under \(\|\cdot\|_{\infty,\supp(\pi_0)}\). Moreover, for every \(R\in\mathcal F_{\mathcal P}\), the map \(a\mapsto R(x,a)\) is continuous on \(S_x\) for \(d_0\)-a.e.\ \(x\in\cX\).

\item \label{as:NP2} \textbf{Reward gap on \(\mathcal P\).}
There exists \(\Delta_{\min}^{\mathcal P}>0\) such that for every truth \(p\in\mathcal P\), \(a_p(X)\) is unique and
\[
R_p\big(X,a_p(X)\big)
-
\sup_{a\in S_X,\ a\neq a_p(X)}
R_p(X,a)
\ge
\Delta_{\min}^{\mathcal P}
\qquad
d_0\text{-a.s.}
\]
Define also \(\Delta_{\max}^{\mathcal P}\coloneqq\sup_{p\in\mathcal P}\sup_{x\in\cX}\sup_{a\in S_x}\big(R_p(x,a_p(x))-R_p(x,a)\big)\).
\end{enumerate}
Condition~\ref{as:NP1} is a compactness-based support-local regularity condition that subsumes the finite-class and bounded-covering-number assumptions commonly used in recent online alignment analyses \citep{xiong2024iterative,ye2024online,wu2025greedy}. Continuity in Condition~\ref{as:NP1} is mild for standard parametric reward classes, including bounded-parameter linear and neural reward families with continuous activations; see Appendix~\ref{app:linear_neural_examples}. Condition~\ref{as:NP2} is a uniform best-response margin assumption: it restricts attention to settings in which the best response is separated from its supported alternatives by a behaviorally meaningful difference, rather than an arbitrarily small difference that users may not reliably distinguish. In the intended LLM setting, \(\cA\) is a discrete set of token sequences rather than a continuous action manifold, so Condition~\ref{as:NP2} should be read as a separation condition on supported candidate responses.\footnote{This interpretation is consistent with models of imprecise valuation and indifference thresholds, under which sufficiently similar alternatives can be difficult to distinguish \citep{woodford2020modeling,cantillo2010thresholds}. It is also natural under Bradley--Terry: a reward difference of $\Delta$ implies preference probability $\sigma(\Delta)$, which approaches $1/2$ as $\Delta\to0$. Thus, near-zero reward gaps correspond to nearly indistinguishable preference signals.}

For any centered reward \(R\), define the truth-centered population loss under truth \(p\) by
\begin{equation}
\mathcal L_p(R)
\coloneqq
\EE\big[\ell(\mathbf v_R,Y)\big],
\label{eq:np_truth_centered_population_loss_main}
\end{equation}
where \(X\sim d_0\), \(\mathbf A=(A_1,\dots,A_K)\sim\pi_0(\cdot\mid X)^{\otimes K}\), \(Y\sim P_{R_p}(\cdot\mid X,\mathbf A)\), and \(\mathbf v_R\coloneqq(R(X,A_1),\dots,R(X,A_K))\).

For the rest of this section, we instantiate the ERM learners from
Section~\ref{sec:np_greedy_alignment_procedure} with
\(\cF=\mathcal F_{\mathcal P}\). In particular,
Condition~\ref{as:NP2} applies to every member of the fitted class, not only
to the realized truth. For data whose slate coordinates lie in the
reference support, Lemma~\ref{lem:mnl_loss_basic} gives
\[
|\widehat{\cL}_n(R)-\widehat{\cL}_n(R')|
\le 2\|R-R'\|_{\infty,\supp(\pi_0)}.
\]
The empirical objective is therefore continuous on the compact class
\(\mathcal F_{\mathcal P}\) and attains its minimum for every \(n\ge1\).
We fix a measurable choice of minimizer. Existence of a minimizer and
measurability of its selection are distinct issues; the latter is part
of our standing convention.

\subsection{Bounded cumulative regret in online RLHF}
\label{sec:np_bounded_regret}

Condition~\ref{as:NP1} implies the finite uniform envelope \(B_{\mathcal P}\coloneqq\sup_{R\in\mathcal F_{\mathcal P}}\|R\|_{\infty,\supp(\pi_0)}<\infty\), and hence \(\Delta_{\max}^{\mathcal P}\le2B_{\mathcal P}\). Since every reward in \(\mathcal F_{\mathcal P}\) is bounded by \(B_{\mathcal P}\) on \(\supp(\pi_0)\), Lemma~\ref{lem:greedy_kl_tilt_lr} in Appendix~\ref{app:proofs_np_online_alignment} implies
$
\beta^{-1}
\le
\frac{d\pi_{t+1}(\cdot\mid x)}{d\pi_0(\cdot\mid x)}(a)
\le
\beta,
\beta\coloneqq e^{2\eta B_{\mathcal P}},
$
for \(d_0\)-a.e.\ \(x\) and \(\pi_0(\cdot\mid x)\)-a.s.\ \(a\).

\begin{lemma}[Automatic isolation of positive regret on \(\mathcal F_{\mathcal P}\)]
\label{prop:np_automatic_positive_regret_isolation}
Assume \ref{as:NP1}--\ref{as:NP2}. Then there exists a constant
\[
\varepsilon_{\mathrm{iso}}^{\mathcal P}>0
\]
such that for every \(p,q\in\mathcal P\),
\[
\mathcal G_p(R_q)\in\{0\}\cup[\varepsilon_{\mathrm{iso}}^{\mathcal P},\infty).
\]
\end{lemma}

\begin{lemma}[Zero truth-centered loss identifies the reward on the reference support]
\label{lem:np_zero_loss_implies_zero_regret}
Assume \ref{as:NP1}. Then for every truth \(p\in\mathcal P\) and every \(R\in\mathcal F_{\mathcal P}\),
\[
\mathcal L_p(R)=\mathcal L_p(R_p)
\quad\Longrightarrow\quad
R(x,a)=R_p(x,a)
\qquad
\forall a\in S_x,
\quad
 d_0\text{-a.e. }x.
\]
In particular,
\[
\mathcal G_p(R)=0.
\]
\end{lemma}

\begin{lemma}[Automatic verification of the truth-centered loss gap]
\label{lem:np_automatic_fixed_scale_gap}
Assume \ref{as:NP1}--\ref{as:NP2}. Fix any \(\varepsilon_0>0\). Then there exists a finite constant
\[
\gamma_{\mathcal P,\varepsilon_0}>0
\]
such that for every pair \(p,q\in\mathcal P\),
\[
\mathcal G_p(R_q)\ge \varepsilon_0
\quad\Longrightarrow\quad
\mathcal L_p(R_q)-\mathcal L_p(R_p)\ge \gamma_{\mathcal P,\varepsilon_0}.
\]
\end{lemma}

\begin{theorem}[Bounded \(O(1)\) regret for online RLHF]
\label{thm:np_bounded_regret}
Assume \ref{as:NP1}--\ref{as:NP2}. Then
\[
\sup_{p\in\mathcal P}\sup_{T\ge 1}
\Regret_{0,p}^{\mathrm{ERM}}(T)<\infty.
\]
More precisely, set \(\varepsilon_0\coloneqq\varepsilon_{\mathrm{iso}}^{\mathcal P}\), \(\gamma\coloneqq\beta^{-K}\gamma_{\mathcal P,\varepsilon_0}\), \(N_\gamma\coloneqq\mathcal N(\mathcal F_{\mathcal P},\gamma/32,\|\cdot\|_{\infty,\supp(\pi_0)})\), and \(c_\gamma\coloneqq\gamma^2/[128(\log K+2B_{\mathcal P})^2]\), where \(\mathcal N(\mathcal F,\varepsilon,\|\cdot\|)\) denotes
the smallest cardinality of an \(\varepsilon\)-net with centers in
\(\mathcal F\), under the indicated norm. Then
\[
\sup_{p\in\mathcal P}\sup_{T\ge 1}
\Regret_{0,p}^{\mathrm{ERM}}(T)
\le
\Delta_{\max}^{\mathcal P}
\left(
1+\left\lceil \frac{1}{c_\gamma}\log(2N_\gamma)\right\rceil
+\frac{1}{e^{c_\gamma}-1}
\right).
\]
\end{theorem}

\subsection{Logarithmic accuracy complexity in offline RLHF}
\label{sec:np_offline_rlhf_result}

The offline result is cleaner than the online result because the logged slates are already sampled from \(\pi_0(\cdot\mid X)^{\otimes K}\), which is exactly the slate distribution used in the truth-centered population loss \(\mathcal L_p\). Therefore the online likelihood-ratio transfer factor \(\beta^{-K}\) is not needed.

The fixed-scale concentration statement underlying the theorem is Lemma~\ref{lem:np_offline_fixed_scale_control} in Appendix~\ref{app:proofs_np_online_alignment}. We state the logarithmic expected-regret consequence as the main offline result.

\begin{theorem}[Offline logarithmic accuracy complexity]
\label{thm:np_offline_log_accuracy}
\label{cor:np_offline_log_accuracy}
Assume \ref{as:NP1}--\ref{as:NP2}. For each truth \(p\in\mathcal P\), let the offline data be generated as in Section~\ref{sec:np_offline_rlhf_procedure} with \(P^\star=P_{R_p}\), and let \(\widehat R_n^{\mathrm{off}}\) be an exact offline ERM over \(\mathcal F_{\mathcal P}\). Set \(\varepsilon_0\coloneqq\varepsilon_{\mathrm{iso}}^{\mathcal P}\), \(\gamma_0\coloneqq\gamma_{\mathcal P,\varepsilon_{\mathrm{iso}}^{\mathcal P}}\), \(N_0^{\mathrm{off}}\coloneqq\mathcal N(\mathcal F_{\mathcal P},\gamma_0/32,\|\cdot\|_{\infty,\supp(\pi_0)})\), and \(c_0^{\mathrm{off}}\coloneqq\gamma_0^2/[128(\log K+2B_{\mathcal P})^2]\). Then, for every \(n\ge1\),
\[
\sup_{p\in\mathcal P}
\EE_p\!\left[\mathcal G_p(\widehat R_n^{\mathrm{off}})\right]
\le
2\Delta_{\max}^{\mathcal P}N_0^{\mathrm{off}}
\exp\{-c_0^{\mathrm{off}}n\}.
\]
If \(\Delta_{\max}^{\mathcal P}=0\), the expected regret is zero
for every sample size. Otherwise, for any \(\eps>0\), exact
reference-logged offline RLHF achieves expected temperature-zero regret
at most \(\eps\) for every integer sample size satisfying
\[
n\ge\max\!\left\{1,
\left\lceil\frac{1}{c_0^{\mathrm{off}}}
\log\!\left(\frac{2\Delta_{\max}^{\mathcal P}N_0^{\mathrm{off}}}{\eps}\right)
\right\rceil\right\}.
\]
For a fixed model class, the offline accuracy complexity is therefore
\(O(\log(1/\eps))\) as \(\eps\downarrow0\).
\end{theorem}

\section{Experiments}
\label{sec:experiment}
To isolate the paper's main point, we use the linear BT testbed of \cite{wu2025greedy} while changing the evaluation target: instead of evaluating full soft-policy recovery, we evaluate the temperature-zero action induced by the learned reward estimate. This directly illustrates the paper's claim that continued KL-regularized regret need not imply continued error in best-response identification.

\subsection{Shared testbed and evaluation}

Both experiments use the same finite-action linear BT model, so their results concern the same best-response identification task. To mirror the positive-gap requirement in Condition~\ref{as:NP2} in the simulation, we first draw a fixed bank \(\mathcal X_{\mathrm{sim}}\) of \(20{,}000\) contexts independently from \(\mathrm{Unif}([0,1]^5)\) and define the simulation context distribution \(d_0\) to be uniform on this finite bank. A fixed set of six actions \(\mathcal A=\{a_1,\ldots,a_6\}\subset[0,1]^5\) and a ground-truth matrix \(W^\star\in[0,1]^{5\times5}\) define the reward
\[
R^\star(x,a)=x^\top W^\star a.
\]
The reference policy \(\pi_0\) is uniform over the six actions. Given a context and a pair \((a_1,a_2)\), the first response is preferred with probability \(\sigma(R^\star(x,a_1)-R^\star(x,a_2))\), where \(\sigma(z)=(1+e^{-z})^{-1}\). We fit a reward of the same bilinear form by numerical BT maximum likelihood with box constraints \(\widehat W\in[0,1]^{5\times5}\), starting from \(\widehat W_0=0\).

Our common evaluation target is the temperature-zero regret of the fitted selector. For a fitted matrix \(\widehat W\), define \(\widehat a(x)=\argmax_{a\in\mathcal A}x^\top\widehat W a\) and evaluate
\[
\EE_X\!\left[R^\star(X,a^\star(X))-R^\star(X,\widehat a(X))\right],
\]
where \(a^\star(x)\) maximizes the true reward. We estimate this expectation on fresh draws from \(d_0\), independent of the preference data, and report means and pointwise standard errors across independent repeats. The online experiment additionally accumulates these one-step regrets over rounds; the offline experiment relates them to the size of a fixed reference-policy log. The selected problem instance has minimum true top-two gap \(0.2703\) over the entire support \(\mathcal X_{\mathrm{sim}}\), so the experimental truth satisfies the positive-gap condition used in Condition~\ref{as:NP2}. This verifies the margin for the selected truth, not
Condition~\ref{as:NP2} for the entire fitted box-constrained class: that
class contains the zero matrix, which ties all actions. Moreover, the
numerical optimizer is not an exact ERM oracle. The experiments
therefore illustrate the qualitative prediction rather than verify all
the assumptions of the theorems. Appendix~\ref{app:experiment_details}
gives the instance-selection procedure and the separate sampling and
optimization settings.

\subsection{Greedy online alignment}

The online experiment asks whether best-response identification can stabilize while feedback collection remains randomized. We initialize \(\pi_1=\pi_0\). At round \(t\), both responses are sampled independently from the current KL-tilted policy
\[
\pi_t(a\mid x_t)\propto\pi_0(a)\exp\!\big(\eta\widehat R_{t-1}(x_t,a)\big),
\qquad
\widehat R_{t-1}(x,a)=x^\top\widehat W_{t-1}a,
\]
exactly as in the online procedure of Section~\ref{sec:np_offline_rlhf_procedure}. After observing the BT preference label, we refit the reward on all comparisons collected so far. We use \(\eta\in\{1,2,3\}\), a horizon of \(T=200\), and \(50\) independent trajectories per value of \(\eta\). Each fitted selector is evaluated on a fresh batch of \(4096\) contexts at every iteration.

Figure~\ref{fig:regret}  illustrates the same qualitative pattern, uniformly across all three regularization levels: the one-step temperature-zero regret falls rapidly to essentially zero well before round 100, and the cumulative regret curves flatten after roughly the first 100 rounds. This is consistent with the bounded $O(1)$ regret behavior predicted by Theorem~\ref{thm:np_bounded_regret}: once the learned reward estimate induces the correct top-ranked action in almost all contexts, the remaining temperature-zero regret becomes zero, so the cumulative curve plateaus rather than continuing to grow throughout the horizon. Note that the online DPO curves are visually indistinguishable from the online RLHF curves at each $\eta$, exactly as predicted by the policy-space equivalence in the pairwise setting discussed earlier.

It is useful to compare this with the results reported in \cite{wu2025greedy}, which studies greedy online RLHF under KL-regularized regret. Their theory yields logarithmic KL-regularized cumulative regret. The comparison highlights our measurement point: the best-response selector can stabilize even while KL-regularized regret continues to evaluate discrepancies in the full randomized policy.

\begin{figure}[ht!]
    \centering
    \includegraphics[width=\linewidth]{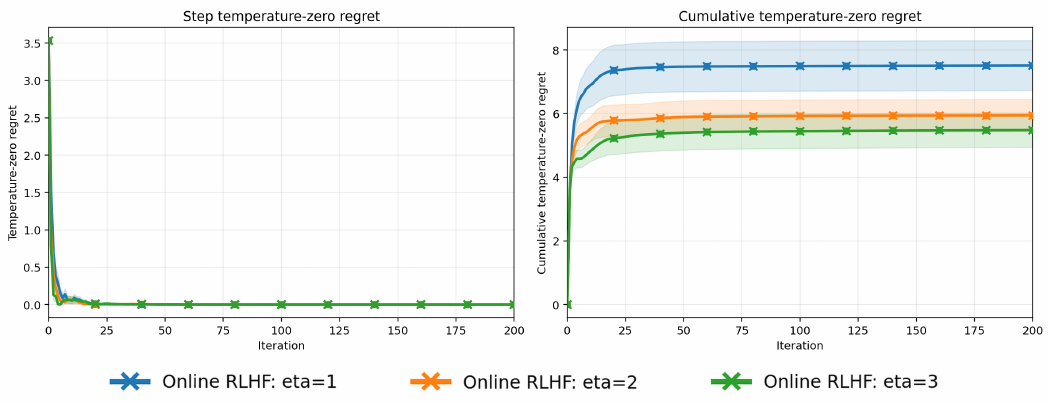}
    \caption{Greedy online RLHF in the finite-action linear Bradley--Terry experiment of \citet{wu2025greedy}, evaluated using temperature-zero regret. Curves show the mean over \(50\) repeated runs, and shaded bands denote the standard error. For each \(\eta\in\{1,2,3\}\), one-step temperature-zero regret (left) rapidly approaches zero and cumulative temperature-zero regret (right) levels off.}
    \label{fig:regret}
\end{figure}

\subsection{Offline AI alignment}

The offline experiment asks whether the same best-response identification task can be learned rapidly from a fixed reference-policy log. Unlike the online experiment, responses are sampled independently from the reference policy $\pi_0$, so the distribution of collected comparisons does not depend on the fitted reward. We repeated the experiment $200$ times; in each repeat, we evaluate all fitted selectors on the same independent batch of $16{,}384$ contexts. Figure~\ref{fig:offline_regret_fits} reports the resulting mean temperature-zero regret for $ n\le1000$.

Figure~\ref{fig:offline_regret_fits} illustrates the qualitative pattern predicted by Theorem~\ref{thm:np_offline_log_accuracy}: temperature-zero regret falls rapidly toward zero as the size of the fixed comparison log increases. To compare the predicted accuracy-complexity forms, we fit exponential decay $Ae^{-cn}$ and inverse decay $A/n$ to the observed mean regret. These correspond, respectively, to $O(\log(1/\epsilon))$ and $O(1/\epsilon)$ sample complexity for reaching target regret $\epsilon$. Over $ n\le1000$, the exponential specification achieves log-space $R^2=0.875$, compared with $0.656$ for inverse decay, and reduces log-space RMSE from $0.793$ to $0.478$, a $39.7\%$ reduction. Thus, over this finite-sample range, the observed decline in temperature-zero regret is more closely described by exponential than inverse decay under the chosen fitting criterion.

It is useful to compare this pattern with the standard inverse accuracy dependence associated with KL-regularized offline alignment guarantees \citep{wu2025greedy}. Theorem~\ref{thm:np_offline_log_accuracy} predicts a sharper result for the decision-relevant task studied here: because temperature-zero regret only requires identifying the correct top-ranked response, expected regret can decrease exponentially with the number of reference-policy comparisons, yielding $O(\log(1/\epsilon))$ rather than $O(1/\epsilon)$ accuracy complexity. The simulation is consistent with this distinction: the fitted reward can identify the best response rapidly even though recovering the full finite-temperature aligned policy is a more demanding objective.

\begin{figure}[ht!]
    \centering
    \begin{minipage}[c]{0.60\linewidth}
        \centering
        \includegraphics[width=\linewidth]{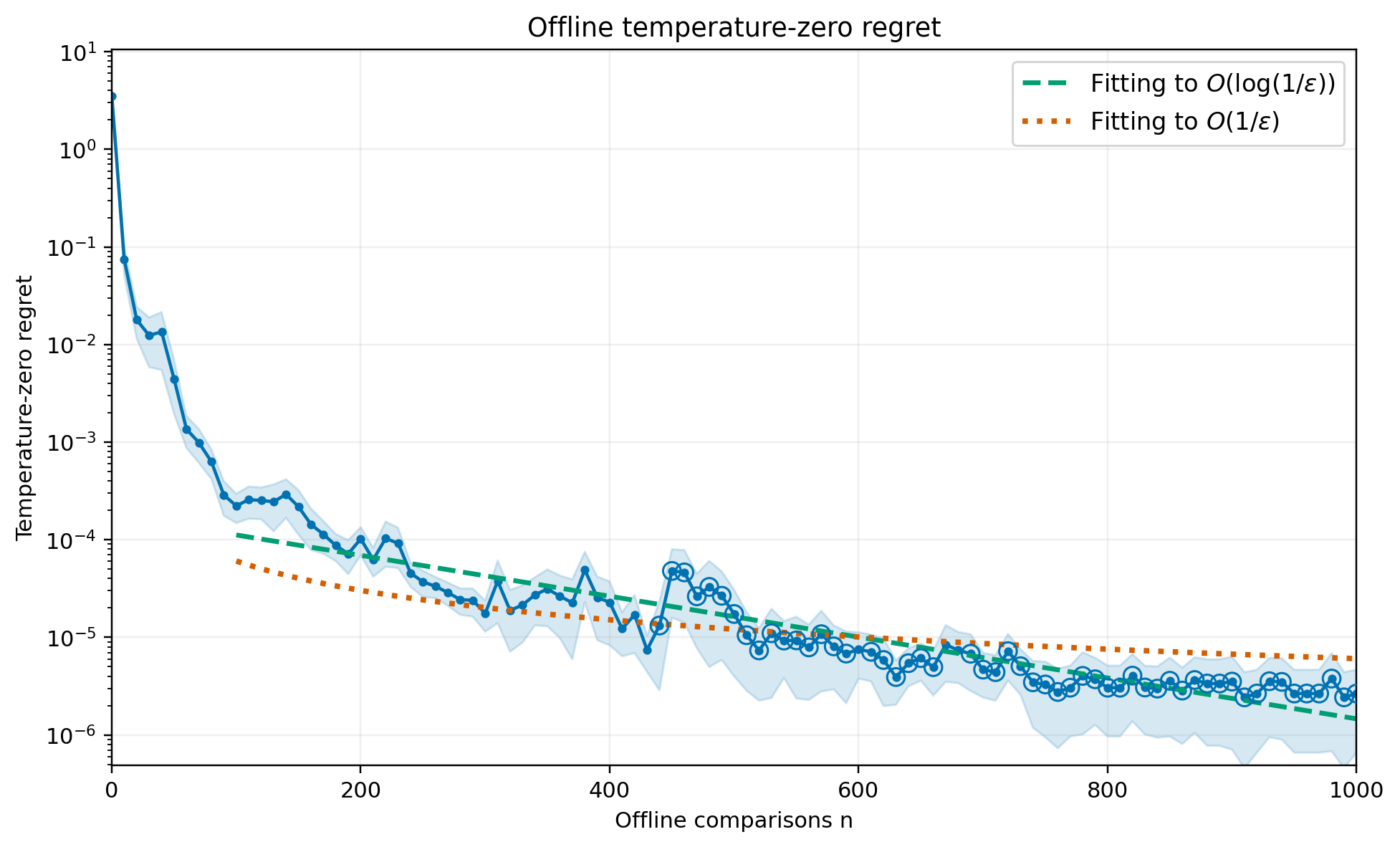}
    \end{minipage}\hfill
    \begin{minipage}[c]{0.38\linewidth}
        \centering
        \footnotesize
        \setlength{\tabcolsep}{3pt}
        \renewcommand{\arraystretch}{1.2}
        \begin{tabular}{@{}lcc@{}}
            \toprule
            \shortstack[l]{Fitted\\complexity}
            & \shortstack{Log-space\\\(R^2\) \(\uparrow\)}
            & \shortstack{Log-space\\RMSE \(\downarrow\)} \\
            \midrule
            \(O(\log(1/\epsilon))\)
            & \textbf{0.875} & \textbf{0.478} \\
            \(O(1/\epsilon)\)
            & 0.656 & 0.793 \\
            \bottomrule
        \end{tabular}
    \end{minipage}

    \caption{Offline temperature-zero regret and fitted accuracy complexity.
    The blue curve is the mean over \(200\) independent offline logs; shading denotes one standard error across repeats. Open circles mark sample sizes with fewer than \(10\) positive-regret repeats. }
    \label{fig:offline_regret_fits}
\end{figure}

\section{Conclusion}

We show that the apparent efficiency of greedy AI alignment is much less mysterious once we distinguish two learning problems that standard KL-regularized regret combines: learning a reward model and recovering the entire randomized aligned policy. Under realizability, compactness, and a uniform best-response margin over the fitted reward class, offline RLHF identifies the optimal response with exponentially decreasing expected temperature-zero regret, while greedy online RLHF incurs only bounded cumulative temperature-zero regret. Thus, the reward model can effectively finish the decision-relevant task of learning what to choose even while the KL-regularized policy continues to adjust probabilities over non-selected responses. Existing logarithmic policy-regret guarantees therefore do not show that greedy alignment must remain slow at learning the best response; they partly reflect the more demanding objective of recovering the full soft policy. What can look like the unreasonable effectiveness of greedy alignment is, under the decision-relevant metric studied here, a consequence of best-response identification being substantially easier than full policy recovery.

% Edit this statement before submission so that it reflects all uses of generative AI.
\subsection*{AI use statement}
Generative AI tools were used to assist with manuscript organization
and editing, expansion and checking of mathematical proofs, \LaTeX{}
formatting, figure-legend editing, and document preparation.

\newpage

\bibliographystyle{plainnat}
\bibliography{bib}

\appendix

\newpage

\section{Related works}\label{sec:related}

\paragraph{Preference-based alignment.}
Early theory for RLHF established statistical guarantees for learning from pairwise and \(K\)-wise comparisons and clarified the role of reward-model estimation in preference-based control \citep{zhu2023principled}; adjacent work in preference-based RL studied finite-time guarantees and compared the intrinsic difficulty of RLHF to standard reward-based RL \citep{xu2020preference,wang2023rlhf}. A complementary unifying view is given by \(\Psi\)PO, which treats RLHF, DPO, and identity-based objectives as special cases of a broader preference-learning framework and makes explicit when scalar-reward reductions are or are not faithful to the underlying preference problem \citep{azar2024paradigm}.

\paragraph{KL-regularized iterative RLHF.}
A recent theory line studies the same KL-regularized contextual-bandit / RLHF objective considered here: \citet{xiong2024iterative} analyze offline, online, and hybrid iterative preference learning under KL constraints, and \citet{ye2024online} extend the online analysis beyond Bradley--Terry to general preference oracles. Subsequent sharp analyses show that KL regularization can fundamentally change the local statistical difficulty, yielding fast rates in offline and online settings under suitable assumptions \citep{zhao2024sharp,zhao2025online,zhao2025offline}.

\paragraph{Greedy and on-policy preference learning.}
Most closely related, \citet{wu2025greedy} show that greedy sampling itself can attain the same \(O(\log T)\) order as confidence-based methods in RLHF, including under general preference models, so explicit optimism or pessimism is not always necessary. \citet{kim2026coverage}, \citet{chen2026sampling}, and \citet{lee2026regularized} refine this picture by analyzing coverage improvement in on-policy DPO, the effect of sampling and reference choices on iterative dynamics, and greedy learning under generalized bilinear or potentially intransitive preference models with strongly convex regularization.

\paragraph{Simple-regret and deployment-centric objectives.}
Methodologically, our temperature-zero criterion is closer to simple-regret evaluation than to standard finite-temperature policy regret, because it scores the final selected action rather than the full stochastic policy used during data collection. This objective has been studied in contextual bandits both in pure-exploration / exploitation formulations and in more recent algorithms that explicitly trade off cumulative and simple regret \citep{deshmukh2018simpleregret,krishnamurthy2023proportional}. Our contribution is therefore not to introduce simple regret itself, but to characterize this decision-centric target for the same greedy KL-regularized preference-learning process used in RLHF/DPO, including adaptive online data collection, and to contrast its rate with full soft-policy recovery. Recent inference-time alignment work similarly studies the top-of-ranking decision more directly, for example by analyzing best-of-\(N\) scaling, coverage, and reward hacking for response selection with a fixed reward model \citep{huang2025bestofn}.

\section{Experiment reproducibility details}
\label{app:experiment_details}

This appendix records the shared problem instance and the separate online and offline runs in Section~\ref{sec:experiment}. The online artifacts correspond to running the attached paper's code.py file:
\begin{verbatim}
python Experiments/code.py \
  --output-dir \
  Experiments/horizon200_gapforced_0p2_seed3500_rep50 \
  --seed 3500 \
  --dimension 5 \
  --num-actions 6 \
  --horizon 200 \
  --repeats 50 \
  --eval-contexts 4096 \
  --mle-maxiter 50 \
  --mle-ftol 1e-9 \
  --min-probe-gap 0.2 \
  --gap-probe-contexts 20000 \
  --problem-search-limit 1000
\end{verbatim}

\paragraph{Problem-instance selection.}
Using the base seed \(3500\), the script first draws a fixed bank \(\mathcal X_{\mathrm{sim}}\) of \(20{,}000\) contexts from \(\mathrm{Unif}([0,1]^5)\). This bank is the support of the simulation context distribution \(d_0\), from which training and evaluation contexts are sampled uniformly with replacement. The script then scans candidate problem seeds \(3500,3501,\dots\) until it finds an instance whose minimum true top-two reward gap over this entire support is at least \(0.2\). For the run used in Figure~\ref{fig:regret}, the accepted candidate is the \(47\)-th one, with chosen problem seed \(3546\). Its minimum support-wide gap is \(0.2703\), and the mean support-wide gap is \(1.9018\). Thus the simulated truth has a uniform positive best-response gap on the support of \(d_0\).

\paragraph{Online data collection and fitting.}
The experiment uses \(k=5\)-dimensional contexts, \(|\mathcal A|=6\) actions, horizon \(T=200\), and regularization coefficients \(\eta\in\{1,2,3\}\). Each context \(x_t\) is drawn independently from the finite-support distribution \(d_0\) defined above. Both actions are sampled independently from the current KL-tilted policy, exactly as in the theoretical online procedure, and the binary preference is then drawn from the Bradley--Terry probability
\[
\sigma\!\big(x_t^\top W^\star(a_{t,1}-a_{t,2})\big).
\]
After each round, the estimate \(\widehat W_t\) is refit on all comparisons collected so far by bounded L-BFGS-B maximum likelihood on the \(25\)-dimensional parameter vector, initialized from the previous iterate, with box constraints \([0,1]\), \texttt{maxiter}\(=50\), and \texttt{ftol}\(=10^{-9}\).

\paragraph{Evaluation and uncertainty reporting.}
For each \(\eta\), we run \(50\) independent trajectories. At every iteration \(t\), the one-step temperature-zero regret is estimated on a fresh Monte Carlo batch of \(4096\) evaluation contexts sampled from \(d_0\), independently of the online trajectory. Figure~\ref{fig:regret} plots the mean across the \(50\) trajectories, and the shaded bands are pointwise standard errors, computed as the sample standard deviation divided by \(\sqrt{50}\). Appendix~\ref{app:dpo_experiment} presents the equivalent pairwise policy-space interpretation of these same trajectories.

\paragraph{Offline data collection, fitting, and evaluation.}
The offline run uses the same saved problem instance and base simulation seed \(4500\). For each of \(200\) repeats, we generate \(5000\) independent reference-policy comparisons and fit only the prefixes \(n=0,10,\ldots,2000\); generating the full log preserves the random streams used in the extended simulation. Both actions are sampled independently and uniformly, and preferences follow the same BT law as above. Each fit uses bounded L-BFGS-B with \texttt{maxiter}\(=1000\), \texttt{ftol}\(=10^{-12}\), and \texttt{gtol}\(=10^{-8}\), initialized from the preceding fitted prefix. At \(n=0\), the reward matrix is zero. Evaluation uses \(16{,}384\) independent draws from \(d_0\) per repeat, shared across prefixes. All \(40{,}000\) nonempty-prefix optimizations reported convergence. The run is reproduced from the project root by
\begin{verbatim}
python offline_experiment.py \
  --seed 4500 --repeats 200 --eval-contexts 16384 \
  --sample-step 10 --max-samples 2000 --data-horizon 5000 \
  --output-dir Experiments/offline_dense_n0_2000_step10_seed4500_rep200
\end{verbatim}
Figure~\ref{fig:offline_regret_fits} uses the first \(1000\) comparisons and fits both decay curves on \(100\le n\le1000\). Their fitted forms are \(1.81386\times10^{-4}\exp(-0.00481934n)\) and \(0.00605738/n\). The reported scores are computed in natural-log regret on these same \(91\) points; all their mean regrets are positive, and no positive-repeat threshold is applied to either fit.

\section{Online DPO as an equivalent pairwise formulation}
\label{app:greedy_alignment_details}

This appendix collects the exact pairwise DPO formulation, its bounded-regret corollary, and the corresponding simulation overlay. The optimization is over the policies induced by the same reward class used in the main text; no separate regret argument is needed.

\subsection{Equivalent policy-space view: greedy online DPO}
\label{sec:np_online_dpo_view}

When the feedback is pairwise (\(K=2\)), the same greedy loop can be written directly in policy coordinates in the style of direct preference optimization (DPO) \citep{guo2024direct}. For the regular reward class \(\cF=\mathcal F_{\mathcal P}\)
used in the theoretical results, define the induced policy class
\[
\Pi_{\cF}^{\mathrm{DPO}}
\coloneqq
\{\pi_R : R\in\cF\},
\]
where \(\pi_R\) is the KL-tilted policy from \eqref{eq:np_greedy_policy_update_general}.

Given pairwise data
\[
(x_s,a_{s,1},a_{s,2},y_s),
\qquad
 y_s\in\{1,2\},
\]
define the exact online DPO empirical loss of a policy \(\pi\in\Pi_{\cF}^{\mathrm{DPO}}\) by
\begin{equation}
\widehat{\cL}^{\mathrm{DPO}}_t(\pi)
\coloneqq
\frac1t\sum_{s=1}^t
-\log \sigma\!\left(
\frac1\eta
\left[
\log\frac{d\pi(\cdot\mid x_s)}{d\pi_0(\cdot\mid x_s)}(a_{s,y_s})
-
\log\frac{d\pi(\cdot\mid x_s)}{d\pi_0(\cdot\mid x_s)}(a_{s,3-y_s})
\right]
\right),
\label{eq:np_dpo_empirical_risk}
\end{equation}
where \(\sigma(u)\coloneqq (1+e^{-u})^{-1}\). Initialize \(\widehat\pi_0=\pi_0\), collect pairwise feedback under the same online sampling rule as in the main text, and refit using all accumulated comparisons. The exact online DPO update is
\begin{align}
    \widehat\pi_t
\in
\argmin_{\pi\in\Pi_{\cF}^{\mathrm{DPO}}}
\widehat{\cL}^{\mathrm{DPO}}_t(\pi),
\qquad
\pi_{t+1}\gets \widehat\pi_t . \label{eq:DPOupdate}
\end{align}

Under exact optimization over this induced policy class, this is
a reparameterization of the reward-space update. We evaluate density
ratios using the canonical version
\(d\pi_R/d\pi_0=e^{\eta R}/Z_R\) on \(S_x\), where
\(Z_R(x)=\int e^{\eta R(x,a')}\,\pi_0(da'\mid x)\).
Thus, at every supported action, this version satisfies
\[
\frac1\eta
\log\frac{d\pi_R(\cdot\mid x)}{d\pi_0(\cdot\mid x)}(a)
=
R(x,a)
-
\frac1\eta
\log\!\int_{\cA}e^{\eta R(x,a')}\,\pi_0(da'\mid x).
\]
Hence the normalization term cancels in pairwise differences, and for every realized example \((x_s,a_{s,1},a_{s,2},y_s)\),
\[
\frac1\eta
\left[
\log\frac{d\pi_R(\cdot\mid x_s)}{d\pi_0(\cdot\mid x_s)}(a_{s,y_s})
-
\log\frac{d\pi_R(\cdot\mid x_s)}{d\pi_0(\cdot\mid x_s)}(a_{s,3-y_s})
\right]
=
R(x_s,a_{s,y_s})-R(x_s,a_{s,3-y_s}).
\]
Therefore
\[
-\log \sigma\!\big(R(x_s,a_{s,y_s})-R(x_s,a_{s,3-y_s})\big)
=
\log\!\Big(e^{R(x_s,a_{s,1})}+e^{R(x_s,a_{s,2})}\Big)
-
R(x_s,a_{s,y_s}),
\]
so
\[
\widehat{\cL}^{\mathrm{DPO}}_t(\pi_R)=\widehat{\cL}_t(R)
\qquad
\text{when }K=2.
\]

\paragraph{Scope of the equivalence.}
The equality of empirical losses holds for every realized comparison history. With common data and a matched measurable choice of exact minimizer, the two formulations also have identical policy iterates and temperature-zero selectors. This is an equivalence between exact optimization problems over matched classes, not a claim that arbitrary approximate implementations return identical iterates.

\subsection{Bounded temperature-zero regret as a corollary}
\label{app:dpo_regret_corollary}

Specialize to pairwise feedback \(K=2\) and set \(\cF=\mathcal F_{\mathcal P}\). Write
\[
\Pi_{\mathcal P}^{\mathrm{DPO}}
\coloneqq \Pi_{\mathcal F_{\mathcal P}}^{\mathrm{DPO}},
\qquad
\Pi_{\mathcal P}^{\mathrm{DPO},0}
\coloneqq \Pi_{\mathcal P}^{\mathrm{DPO}}\cup\{\pi_0\}.
\]
For \(\pi\in\Pi_{\mathcal P}^{\mathrm{DPO}}\), choose any \(R\in\mathcal F_{\mathcal P}\) such that \(\pi=\pi_R\) and set \(a_\pi(x)\coloneqq a_R(x)\). This is well defined \(d_0\)-a.s.\ because any two such rewards differ on \(S_x\) only by an action-independent constant, as shown in the proof below. For the initial policy \(\pi_0\), use \(a_{\pi_0}(x)\coloneqq a_0(x)\), with the same measurable tie-breaking convention as in the main text. Define
\[
\mathcal G_p^{\mathrm{DPO}}(\pi)
\coloneqq
\EE_{X\sim d_0}\!\left[
R_p(X,a_p(X))-R_p(X,a_\pi(X))
\right].
\]
Let \((\widehat\pi_t)_{t\ge0}\) be the exact online DPO iterates
from \eqref{eq:DPOupdate}, initialized at \(\widehat\pi_0=\pi_0\).
Use a measurable representing-reward rule, so that for \(t\ge1\),
\(\widehat\pi_t=\pi_{\widehat R_t^{\mathrm{DPO}}}\) with
\(\widehat R_t^{\mathrm{DPO}}\in\mathcal F_{\mathcal P}\).
A matched measurable reward-space ERM selection, as in the preceding
subsection, provides such a rule.

\begin{corollary}[Bounded temperature-zero regret for online DPO]
\label{cor:np_bounded_regret_dpo}
\label{thm:np_bounded_regret_dpo} % Backward-compatible label.
Assume \ref{as:NP1}--\ref{as:NP2}, and specialize to pairwise feedback \(K=2\). Then
\[
\sup_{p\in\mathcal P}\sup_{T\ge 1}
\sum_{t=0}^{T-1}\EE_p\!\left[\mathcal G_p^{\mathrm{DPO}}(\widehat\pi_t)\right]
<
\infty.
\]
More precisely, with the same constants \(\varepsilon_0,\gamma_{\mathcal P,\varepsilon_0},\gamma,N_\gamma,c_\gamma\) as in Theorem~\ref{thm:np_bounded_regret}, specialized to \(K=2\),
\[
\sup_{p\in\mathcal P}\sup_{T\ge 1}
\sum_{t=0}^{T-1}\EE_p\!\left[\mathcal G_p^{\mathrm{DPO}}(\widehat\pi_t)\right]
\le
\Delta_{\max}^{\mathcal P}
\left(
1+\left\lceil \frac{1}{c_\gamma}\log(2N_\gamma)\right\rceil
+\frac{1}{e^{c_\gamma}-1}
\right).
\]
\end{corollary}

\begin{proof}[Proof of Corollary~\ref{cor:np_bounded_regret_dpo}]
Work on the common full-measure context set supplied by
Lemma~\ref{lem:common_context_set}. Suppose \(R,R'\in\mathcal F_{\mathcal P}\)
induce the same policy at a context in this set. Their densities with
respect to \(\pi_0(\cdot\mid x)\) agree almost surely, and therefore
\[
\frac{e^{\eta R(x,a)}}{Z_R(x)}
=\frac{e^{\eta R'(x,a)}}{Z_{R'}(x)}
\quad\text{for }\pi_0(\cdot\mid x)\text{-almost every }a.
\]
Both sides are positive, so taking logarithms and rearranging gives
\[
R(x,a)-R'(x,a)
=\frac{\log Z_R(x)-\log Z_{R'}(x)}{\eta}
\quad\pi_0(\cdot\mid x)\text{-almost surely}.
\]
The right-hand side is constant in \(a\). Integrating and using
centering of both rewards shows that this constant is zero.
Their continuous action sections then agree on all of \(S_x\), by
Lemma~\ref{lem:support_upgrade}. Thus the induced selector is
independent of the representing reward, up to a \(d_0\)-null set.
The same argument with \(R'=0\) shows that if \(\pi_0\) already belongs
to the induced policy class, its representing reward is zero on the
support at almost every context. Its selector then agrees with the
initial selector \(a_0\) under the common tie-breaking convention.
The definition on the union with \(\{\pi_0\}\) is therefore consistent.

For any supported pair and observed preferred index \(y\in\{1,2\}\),
the canonical densities satisfy
\begin{align*}
\frac1\eta\left[
 \log\frac{d\pi_R}{d\pi_0}(a_y\mid x)
 -\log\frac{d\pi_R}{d\pi_0}(a_{3-y}\mid x)\right]
&=\frac1\eta\bigl[\eta R(x,a_y)-\log Z_R(x)
 -\eta R(x,a_{3-y})+\log Z_R(x)\bigr]\\
&=R(x,a_y)-R(x,a_{3-y}).
\end{align*}
Moreover,
\begin{align*}
-\log\sigma(R(x,a_y)-R(x,a_{3-y}))
&=\log\!\left(1+e^{R(x,a_{3-y})-R(x,a_y)}\right)\\
&=\log\!\left(e^{R(x,a_y)}+e^{R(x,a_{3-y})}\right)-R(x,a_y)\\
&=\ell(\mathbf v_R(x,(a_1,a_2)),y).
\end{align*}
Summing over observations proves
\(\widehat{\mathcal L}^{\mathrm{DPO}}_t(\pi_R)
=\widehat{\mathcal L}_t(R)\) simultaneously over the fitted class on
an almost-sure set of histories. The common context set ensures that
this statement is not affected by candidate-dependent exceptional
contexts.

It follows in both directions that the minimizer sets correspond:
\[
\argmin_{\pi\in\Pi_{\mathcal P}^{\mathrm{DPO}}}
   \widehat{\mathcal L}^{\mathrm{DPO}}_t(\pi)
=\left\{\pi_R:
 R\in\argmin_{U\in\mathcal F_{\mathcal P}}
                  \widehat{\mathcal L}_t(U)\right\}.
\]
Indeed, a reward minimizer has no policy competitor with a smaller
loss, because every policy competitor has a representing reward with
the same loss. Conversely, if an exact policy minimizer has a
representing reward \(R\) that is not an empirical reward minimizer,
a reward with strictly smaller empirical loss would induce a policy
with strictly smaller DPO loss, a contradiction.

Consequently the measurable representing rewards
\(\widehat R_t^{\mathrm{DPO}}\) form an exact ERM sequence. Their
sampling policies are \(\pi_1=\pi_0\) and
\(\pi_{t+1}=\pi_{\widehat R_t^{\mathrm{DPO}}}\), so this sequence
satisfies the online RLHF assumptions. Also, including the initialization,
\[
\mathcal G_p^{\mathrm{DPO}}(\widehat\pi_t)
=\mathcal G_p(\widehat R_t^{\mathrm{DPO}}),
\qquad \widehat R_0^{\mathrm{DPO}}=0.
\]
Applying Theorem~\ref{thm:np_bounded_regret} to this exact ERM rule,
with \(K=2\), gives the stated bound after summing over \(t\).
If the same minimizer choices and sampling randomness are used in
both formulations, equality of their policies holds initially and,
by the loss identity, is preserved after every common observation;
induction then gives identical trajectories. Such matched choices are
needed for equality of particular trajectories, but the regret bound
holds for any measurable representing-reward rule satisfying the exact
optimization requirement.
\end{proof}

\subsection{Simulation overlay and interpretation}
\label{app:dpo_experiment}

Figure~\ref{fig:dpo_overlay} retains the paired RLHF--DPO presentation of the simulation used in Section~\ref{sec:experiment}. In this pairwise implementation, the DPO overlay uses exactly the same trajectory data as the RLHF curves. The reported maximum absolute RLHF-versus-DPO difference is therefore zero for both the one-step and cumulative curves at every \(\eta\), by construction. The overlay illustrates the equivalent policy-space representation; it is not a separate optimization experiment or an independent empirical comparison. Figure~\ref{fig:regret} in the main text labels each shared trajectory only as online RLHF.

\begin{figure}[htbp]
    \centering
    \includegraphics[width=\linewidth]{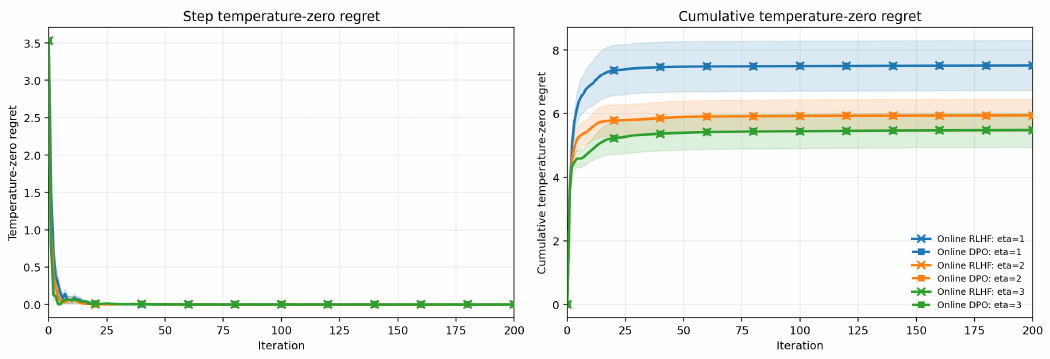}
    \caption{Online RLHF and its equivalent pairwise DPO representation for the simulation in Section~\ref{sec:experiment}. Both labels use the same trajectory data, so the curves coincide by construction at each \(\eta\in\{1,2,3\}\). Curves show the mean over \(50\) repeated runs, and shaded bands denote the standard error. The left panel shows one-step temperature-zero regret; the right panel shows cumulative temperature-zero regret.}
    \label{fig:dpo_overlay}
\end{figure}

\FloatBarrier

\section{Proofs for Section~\ref{sec:nonpersonalized_online_alignment}}
\label{app:proofs_np_online_alignment}

We first record examples of reward families satisfying Condition~\ref{as:NP1}.

\subsection{Linear and neural reward families}
\label{app:linear_neural_examples}

Write
\[
S_x\coloneqq \supp(\pi_0(\cdot\mid x)),
\qquad
S_{\pi_0}\coloneqq \{(x,a)\in \cX\times \cA:\ a\in S_x\}.
\]

\begin{proposition}[Centered linear reward classes satisfy Condition~\ref{as:NP1}]
\label{prop:linear_class_compact_continuous}
Let \(\Theta\subset \R^d\) be compact. Let \(\phi:\cX\times\cA\to\R^d\) be measurable and satisfy
\[
M_\phi
\coloneqq
\sup_{(x,a)\in S_{\pi_0}}\|\phi(x,a)\|_2
<\infty.
\]
Assume also that for \(d_0\)-a.e.\ \(x\), the map \(a\mapsto \phi(x,a)\) is continuous on \(S_x\). For \(\theta\in\Theta\), define the centered linear reward
\[
R_\theta(x,a)
\coloneqq
\theta^\top \phi(x,a)
-
\EE_{A\sim \pi_0(\cdot\mid x)}[\theta^\top \phi(x,A)].
\]
Then the class \(\mathcal F_\Theta^{\mathrm{lin}}\coloneqq\{R_\theta:\theta\in\Theta\}\) is compact under \(\|\cdot\|_{\infty,\supp(\pi_0)}\), and for every \(\theta\in\Theta\), the map \(a\mapsto R_\theta(x,a)\) is continuous on \(S_x\) for \(d_0\)-a.e.\ \(x\).
\end{proposition}

\begin{proof}
Define the reference feature mean
\[
\bar\phi(x)\coloneqq\int_{\cA}\phi(x,a)\,\pi_0(da\mid x).
\]
Each coordinate is integrable because \(\|\phi(x,a)\|_2\le M_\phi\)
on the reference support, which has full reference measure by
Lemma~\ref{lem:support_upgrade}. The kernel integral is measurable, and
\[
\|\bar\phi(x)\|_2
\le\int_{\cA}\|\phi(x,a)\|_2\,\pi_0(da\mid x)
\le M_\phi.
\]
Consequently,
\[
R_\theta(x,a)=\theta^\top\bigl(\phi(x,a)-\bar\phi(x)\bigr),
\qquad
\int R_\theta(x,a)\,\pi_0(da\mid x)=0.
\]
For every context at which the action section of \(\phi\) is continuous,
the action section of \(R_\theta\) is continuous, since the vector
\(\bar\phi(x)\) does not depend on \(a\).

For \(\theta,\tilde\theta\in\Theta\) and \((x,a)\in S_{\pi_0}\),
the Cauchy--Schwarz and triangle inequalities give
\begin{align*}
|R_\theta(x,a)-R_{\tilde\theta}(x,a)|
&=\left|(\theta-\tilde\theta)^\top
                (\phi(x,a)-\bar\phi(x))\right|\\
&\le\|\theta-\tilde\theta\|_2
       \bigl(\|\phi(x,a)\|_2+\|\bar\phi(x)\|_2\bigr)\\
&\le 2M_\phi\|\theta-\tilde\theta\|_2.
\end{align*}
Taking the supremum over \(S_{\pi_0}\) proves continuity of the map
\(\theta\mapsto R_\theta\) into the supportwise supremum norm.
To verify compactness directly, let \((R_{\theta_n})_{n\ge1}\) be any
sequence in the class. Compactness of \(\Theta\) supplies a subsequence
\(\theta_{n_j}\to\theta\in\Theta\). The displayed bound then implies
\(\|R_{\theta_{n_j}}-R_\theta\|_{\infty,\supp(\pi_0)}\to0\).
Thus every sequence has a subsequence converging to a member of the
class. Sequential compactness is equivalent to compactness in a metric
space, which proves the claim.
\end{proof}

\begin{proposition}[Centered neural reward classes satisfy Condition~\ref{as:NP1}]
\label{prop:neural_class_compact_continuous}
Let \(\Theta\subset\R^m\) be compact. For each \(\theta\in\Theta\), let \(f_\theta:S_{\pi_0}\to\R\) be a bounded measurable function. Assume that: (i) for \(d_0\)-a.e.\ \(x\) and every \(\theta\in\Theta\), the map \(a\mapsto f_\theta(x,a)\) is continuous on \(S_x\); and (ii) the parameter-to-function map \(\theta\mapsto f_\theta\) is continuous from \(\Theta\) to \(\ell_\infty(S_{\pi_0})\). Define
\[
R_\theta(x,a)
\coloneqq
f_\theta(x,a)
-
\EE_{A\sim \pi_0(\cdot\mid x)}[f_\theta(x,A)].
\]
Then \(\mathcal F_\Theta^{\mathrm{nn}}\coloneqq\{R_\theta:\theta\in\Theta\}\) is compact under \(\|\cdot\|_{\infty,\supp(\pi_0)}\), and every action section is continuous on \(S_x\) for \(d_0\)-a.e.\ \(x\).
\end{proposition}

\begin{proof}
For a bounded measurable function \(f\) on \(S_{\pi_0}\), define
\[
(Cf)(x,a)\coloneqq f(x,a)-\int f(x,b)\,\pi_0(db\mid x).
\]
The integral is finite, its dependence on \(x\) is measurable, and
\(\int (Cf)(x,a)\,\pi_0(da\mid x)=0\). If the action section of \(f\)
is continuous, then the action section of \(Cf\) is continuous because
the subtracted term is constant in \(a\). For any two bounded functions
\(f,g\),
\begin{align*}
|(Cf)(x,a)-(Cg)(x,a)|
&\le |f(x,a)-g(x,a)|
  +\int |f(x,b)-g(x,b)|\,\pi_0(db\mid x)\\
&\le 2\|f-g\|_{\infty,\supp(\pi_0)}.
\end{align*}
Hence \(C\) is a continuous linear map, with operator norm at most \(2\).
By the assumed continuity of \(\theta\mapsto f_\theta\), its composition
\(\theta\mapsto Cf_\theta=R_\theta\) is continuous.

For completeness, if \((R_{\theta_n})_{n\ge1}\) is a sequence in the
centered class, compactness of \(\Theta\) gives a subsequence
\(\theta_{n_j}\to\theta\in\Theta\). Assumption (ii) and the preceding
bound yield
\[
\|R_{\theta_{n_j}}-R_\theta\|_{\infty,\supp(\pi_0)}
\le2\|f_{\theta_{n_j}}-f_\theta\|_{\infty,\supp(\pi_0)}\longrightarrow0.
\]
The limit belongs to the centered class, so that class is compact.
Assumption (i) and the first paragraph give the claimed continuity of
its action sections.
\end{proof}

\subsection{Basic technical lemmas}

\begin{lemma}[\(\pi_0\)-centering is w.l.o.g.\ for MNL likelihoods]
\label{lem:pi0_centering_wlog_psi}
Fix a measurable reward \(R:\cX\times\cA\to\R\). At every context
under consideration, assume
\(\int |R(x,a)|\,\pi_0(da\mid x)<\infty\); for the policy identity,
also assume \(0<\int e^{\eta R(x,a)}\,\pi_0(da\mid x)<\infty\).
These requirements hold for bounded rewards. Define its
\(\pi_0\)-centered version
\[
\bar R(x,a)\coloneqq R(x,a)-m_R(x),
\qquad
m_R(x)\coloneqq \EE_{a'\sim\pi_0(\cdot\mid x)}[R(x,a')].
\]
Then for every \(x\) and every slate \(\mathbf a\): (i) \(P_R(\cdot\mid x,\mathbf a)=P_{\bar R}(\cdot\mid x,\mathbf a)\); (ii) \(\pi_R(\cdot\mid x)=\pi_{\bar R}(\cdot\mid x)\) for KL tilts; and (iii) \(\argmax_{a\in S_x}R(x,a)=\argmax_{a\in S_x}\bar R(x,a)\).
\end{lemma}

\begin{proof}
Fix a context satisfying the stated integrability conditions, and write
\(m=m_R(x)\). Since \(m\) is finite,
\[
\int\bar R(x,a)\,\pi_0(da\mid x)
=\int R(x,a)\,\pi_0(da\mid x)-m\int1\,\pi_0(da\mid x)=0.
\]
For every slate \(\mathbf a=(a_1,\ldots,a_K)\) and every index \(k\),
\[
P_{\bar R}(k\mid x,\mathbf a)
=\frac{e^{R(x,a_k)-m}}{\sum_{\ell=1}^K e^{R(x,a_\ell)-m}}
=\frac{e^{-m}e^{R(x,a_k)}}{e^{-m}\sum_{\ell=1}^K e^{R(x,a_\ell)}}
=P_R(k\mid x,\mathbf a).
\]
Thus the choice probabilities, and hence their negative log-likelihoods,
are unchanged.

Let \(Z_R(x)=\int e^{\eta R(x,a)}\,\pi_0(da\mid x)\). Then
\[
Z_{\bar R}(x)
=\int e^{\eta(R(x,a)-m)}\,\pi_0(da\mid x)
=e^{-\eta m}Z_R(x).
\]
Substitution into the density formula gives
\[
\frac{d\pi_{\bar R}(\cdot\mid x)}{d\pi_0(\cdot\mid x)}(a)
=\frac{e^{\eta R(x,a)-\eta m}}{e^{-\eta m}Z_R(x)}
=\frac{d\pi_R(\cdot\mid x)}{d\pi_0(\cdot\mid x)}(a),
\]
so the two policy measures are equal. Finally, for any \(a,b\in S_x\),
\(\bar R(x,a)\ge\bar R(x,b)\) holds if and only if
\(R(x,a)\ge R(x,b)\). The maximizing sets therefore agree, and the
common tie-breaking convention gives the same selector.
\end{proof}

\begin{lemma}[Boundedness and Lipschitzness of the MNL log-loss]
\label{lem:mnl_loss_basic}
Let \(\ell(\mathbf v,y)=\log(\sum_{k=1}^K e^{v_k})-v_y\) be the MNL negative log-likelihood. If \(\max_k v_k-\min_k v_k\le 2B\), then for every \(y\),
\[
0\le \ell(\mathbf v,y)\le \log K+2B.
\]
Moreover, for any \(\mathbf v,\mathbf v'\in\R^K\) and any \(y\),
\[
|\ell(\mathbf v,y)-\ell(\mathbf v',y)|\le 2\|\mathbf v-\mathbf v'\|_\infty.
\]
\end{lemma}

\begin{proof}
For every \(y\),
\[
\ell(\mathbf v,y)
=\log\!\left(\sum_{k=1}^K e^{v_k-v_y}\right).
\]
The summand with \(k=y\) equals \(1\), so the expression is nonnegative.
Writing \(m=\min_k v_k\) and \(M=\max_k v_k\), every summand is at most
\(e^{M-m}\). Therefore
\[
\ell(\mathbf v,y)
\le\log\!\left(K e^{M-m}\right)
=\log K+M-m\le\log K+2B.
\]

To prove the Lipschitz bound without a differentiability argument, set
\(d=\|\mathbf v-\mathbf v'\|_\infty\). For each \(k\),
\(v_k' -d\le v_k\le v_k'+d\), whence
\[
e^{-d}\sum_k e^{v_k'}\le\sum_k e^{v_k}\le e^d\sum_k e^{v_k'}.
\]
Taking logarithms gives
\[
\left|\log\sum_k e^{v_k}-\log\sum_k e^{v_k'}\right|\le d.
\]
Since \(|v_y-v_y'|\le d\), the triangle inequality now yields
\[
|\ell(\mathbf v,y)-\ell(\mathbf v',y)|\le d+d
=2\|\mathbf v-\mathbf v'\|_\infty.
\]
\end{proof}

\begin{lemma}[One-step excess loss equals choice-model KL]
\label{lem:excess_is_kl_support}
For fixed \((x,\mathbf a)\), truth reward \(R^\star\), and candidate reward \(R\), let \(p^\star=P_{R^\star}(\cdot\mid x,\mathbf a)\) and \(p=P_R(\cdot\mid x,\mathbf a)\). Then
\[
\EE_{y\sim p^\star}\big[\ell(\mathbf v_R,y)-\ell(\mathbf v_{R^\star},y)\big]
=
\KL(p^\star\|p).
\]
\end{lemma}

\begin{proof}
Every softmax coordinate is strictly positive. By the definition of
\(\ell\), \(\ell(\mathbf v_R,k)=-\log p_k\) and
\(\ell(\mathbf v_{R^\star},k)=-\log p_k^\star\). Consequently,
\begin{align*}
\EE_{Y\sim p^\star}
 [\ell(\mathbf v_R,Y)-\ell(\mathbf v_{R^\star},Y)]
&=\sum_{k=1}^K p_k^\star(-\log p_k+\log p_k^\star)\\
&=\sum_{k=1}^K p_k^\star\log\frac{p_k^\star}{p_k}
=\KL(p^\star\|p).
\end{align*}
We will also use nonnegativity and its equality case. The scalar inequality
\(-\log z\ge1-z\), with equality exactly when \(z=1\), gives
\[
\KL(p^\star\|p)
\ge\sum_{k=1}^K p_k^\star\left(1-\frac{p_k}{p_k^\star}\right)
=\sum_{k=1}^K(p_k^\star-p_k)=0.
\]
Because all weights \(p_k^\star\) are positive, equality holds if and
only if \(p_k=p_k^\star\) for every \(k\).
\end{proof}

\begin{lemma}[Slate expectation domination]
\label{lem:slate_domination_betaK}
Fix \(x\) and a distribution \(\pi(\cdot\mid x)\ll\pi_0(\cdot\mid x)\)
whose Radon--Nikodym derivative \(r_x=d\pi(\cdot\mid x)/d\pi_0(\cdot\mid x)\)
satisfies \(r_x\ge\beta^{-1}\) \(\pi_0(\cdot\mid x)\)-almost surely. Let \(\mathbf A\sim \pi(\cdot\mid x)^{\otimes K}\) and \(\mathbf A_0\sim \pi_0(\cdot\mid x)^{\otimes K}\). Then for every nonnegative measurable \(g\),
\[
\EE[g(\mathbf A)]\ge \beta^{-K}\EE[g(\mathbf A_0)].
\]
\end{lemma}

\begin{proof}
Write \(\mu=\pi_0(\cdot\mid x)\), \(\nu=\pi(\cdot\mid x)\), and
\(r=d\nu/d\mu\). For every measurable rectangle
\(E_1\times\cdots\times E_K\), iterated integration gives
\[
\int_{E_1\times\cdots\times E_K}\prod_{k=1}^K r(a_k)\,
                   \mu^{\otimes K}(d\mathbf a)
=\prod_{k=1}^K\int_{E_k}r(a)\,\mu(da)
=\prod_{k=1}^K\nu(E_k).
\]
The measure with density \(\prod_{k=1}^K r(a_k)\) relative to
\(\mu^{\otimes K}\) is thus \(\nu^{\otimes K}\), since finite measures
agreeing on measurable rectangles agree on the product \(\sigma\)-field.
Outside a \(\mu^{\otimes K}\)-null set, every coordinate satisfies
\(r(a_k)\ge\beta^{-1}\), and hence the product is at least
\(\beta^{-K}\). For a nonnegative measurable \(g\), including one with
infinite expectation, it follows that
\[
\EE[g(\mathbf A)]
=\int g(\mathbf a)\prod_{k=1}^K r(a_k)\,\mu^{\otimes K}(d\mathbf a)
\ge\beta^{-K}\int g(\mathbf a)\,\mu^{\otimes K}(d\mathbf a)
=\beta^{-K}\EE[g(\mathbf A_0)].
\]
\end{proof}

\begin{lemma}[KL-tilt optimality and likelihood-ratio bounds]
\label{lem:greedy_kl_tilt_lr}
Fix \(\eta>0\). Let \(R\) satisfy \(|R(x,a)|\le B\) for \(\pi_0(\cdot\mid x)\)-a.e.\ \(a\). Then the KL tilt \(\pi_R(a\mid x)\propto\pi_0(a\mid x)e^{\eta R(x,a)}\) satisfies
\[
e^{-2\eta B}
\le
\frac{d\pi_R(\cdot\mid x)}{d\pi_0(\cdot\mid x)}(a)
\le
 e^{2\eta B}
\qquad \pi_0(\cdot\mid x)\text{-a.s.}
\]
Moreover, \(\pi_R(\cdot\mid x)\) is the unique maximizer, over
\(\pi\ll\pi_0(\cdot\mid x)\), of
\(\int R(x,a)\,\pi(da)-\eta^{-1}\KL(\pi\|\pi_0(\cdot\mid x))\).
Writing \(Z_R(x)=\int e^{\eta R(x,a)}\,\pi_0(da\mid x)\), its value is
\(\eta^{-1}\log Z_R(x)\), and for every such \(\pi\),
\begin{equation}
\int R(x,a)\,\pi(da)-\eta^{-1}\KL(\pi\|\pi_0(\cdot\mid x))
=\eta^{-1}\log Z_R(x)-\eta^{-1}\KL(\pi\|\pi_R(\cdot\mid x)).
\label{eq:kl_tilt_variational_identity}
\end{equation}
The identity is interpreted in the extended real sense when a KL
divergence is infinite.
\end{lemma}

\begin{proof}
Fix the context and abbreviate \(\mu=\pi_0(\cdot\mid x)\) and
\(Z=Z_R(x)\). Exponentiating the reward bound and integrating against
the probability measure \(\mu\) gives
\[
e^{-\eta B}\le e^{\eta R(x,a)}\le e^{\eta B}
\quad\mu\text{-a.s.},
\qquad
e^{-\eta B}\le Z\le e^{\eta B}.
\]
Thus \(Z\) is finite and strictly positive. The function
\(r(a)=e^{\eta R(x,a)}/Z\) integrates to one and satisfies
\[
\frac{e^{-\eta B}}{e^{\eta B}}
\le r(a)\le\frac{e^{\eta B}}{e^{-\eta B}},
\]
which proves the stated density bounds. In particular, \(r>0\)
\(\mu\)-almost surely, so \(\pi_R\) and \(\mu\) have the same null sets.

Let \(\pi\ll\mu\) and put \(f=d\pi/d\mu\). Since \(r>0\),
\(d\pi/d\pi_R=f/r\), and
\(\log r=\eta R(x,\cdot)-\log Z\). The latter function is bounded.
If \(\KL(\pi\|\mu)<\infty\), substitution into the definition of KL
therefore yields
\begin{align*}
\KL(\pi\|\pi_R)
&=\int\log\frac{f(a)}{r(a)}\,\pi(da)\\
&=\int\log f(a)\,\pi(da)
   -\eta\int R(x,a)\,\pi(da)+\log Z\\
&=\KL(\pi\|\mu)-\eta\int R(x,a)\,\pi(da)+\log Z.
\end{align*}
The rearranged equality is \eqref{eq:kl_tilt_variational_identity}.
If \(\KL(\pi\|\mu)=\infty\), subtracting the bounded function
\(\log r\) leaves the KL divergence infinite; both objective expressions
in \eqref{eq:kl_tilt_variational_identity} then equal \(-\infty\).
Here the negative part of \(\int\log f\,d\pi=\int f\log f\,d\mu\)
is finite because \(-u\log u\le e^{-1}\) on \(0\le u\le1\).

Finally, for a probability measure \(\nu\) and \(\pi\ll\nu\) with
density \(g\), the inequality \(u\log u-u+1\ge0\) gives
\[
\KL(\pi\|\nu)
=\int(g\log g-g+1)\,d\nu\ge0,
\]
with equality only when \(g=1\) \(\nu\)-almost surely. Applying this
fact with \(\nu=\pi_R\) to \eqref{eq:kl_tilt_variational_identity}
shows that the maximum is \(\eta^{-1}\log Z\) and is attained only at
\(\pi=\pi_R\).
\end{proof}

\begin{lemma}[Supportwise upgrade of almost-sure equality]
\label{lem:support_upgrade}
Let \((\cA,d)\) be a separable metric space, let \(\mu\) be a Borel probability measure on \(\cA\), and let \(S=\supp(\mu)\). If continuous functions \(f,g:S\to\R\) satisfy \(f=g\) \(\mu\)-a.s., then \(f=g\) on \(S\).
\end{lemma}

\begin{proof}
A separable metric space has a countable base of open sets; fix one and
denote it by \((U_j)_{j\ge1}\). By the definition of support, a point
\(a\notin S\) has an open neighborhood of \(\mu\)-measure zero. A base
element containing \(a\) and contained in that neighborhood also has
measure zero. Conversely, every point of a zero-measure open set lies
outside the support. Hence
\[
\cA\setminus S=\bigcup_{j:\,\mu(U_j)=0}U_j,
\qquad
\mu(\cA\setminus S)=0.
\]
The last equality uses countable subadditivity.

Set \(h=f-g\), which is continuous on \(S\). Suppose for a contradiction
that \(h(a_0)\ne0\) for some \(a_0\in S\), and put
\(\delta=|h(a_0)|/2>0\). Relative continuity gives an open set
\(V\subseteq\cA\) containing \(a_0\) such that
\(|h(a)-h(a_0)|<\delta\) for all \(a\in V\cap S\). The reverse triangle
inequality implies \(|h(a)|>\delta\) on \(V\cap S\). Since
\(a_0\in S\), \(\mu(V)>0\); since \(\mu(\cA\setminus S)=0\),
\(\mu(V\cap S)=\mu(V)>0\). Thus \(h\) is nonzero on a set of positive
measure, contradicting \(h=0\) \(\mu\)-almost surely. No such \(a_0\)
exists, and \(f=g\) on all of \(S\).
\end{proof}

\begin{lemma}[A common regularity set for the compact reward class]
\label{lem:common_context_set}
Under Condition~\ref{as:NP1} and the centering convention, there is a
measurable set \(\cX_0\) of \(d_0\)-measure one on which every reward in
\(\mathcal F_{\mathcal P}\) is centered and has a continuous action
section. Under Condition~\ref{as:NP2}, \(\cX_0\) can also be chosen so
that every reward in the class has its stated unique maximizer and
margin on this same set.
\end{lemma}

\begin{proof}
The compact metric class has a countable dense subset
\(\{R_j:j\ge1\}\). For each \(j\), choose a full-measure measurable
set on which \(R_j\) is centered and its action section is continuous;
when Condition~\ref{as:NP2} is assumed, also impose its margin property.
The countable intersection of these sets, denoted by \(\cX_0\), has
measure one.

Fix \(x\in\cX_0\) and \(R\in\mathcal F_{\mathcal P}\). Choose a
sequence \(R_{j_n}\to R\) in the supportwise supremum norm. Its action
sections at \(x\) converge uniformly on \(S_x\). The uniform limit of
continuous functions is continuous, so the action section of \(R\)
is continuous. Also,
\[
\left|\int R(x,a)\,\pi_0(da\mid x)\right|
=\left|\int(R-R_{j_n})(x,a)\,\pi_0(da\mid x)\right|
\le\|R-R_{j_n}\|_{\infty,\supp(\pi_0)}\longrightarrow0,
\]
which proves centering.

Suppose in addition that Condition~\ref{as:NP2} holds. Since the
approximating sequence is Cauchy, for some \(n_0\) and all
\(n,m\ge n_0\),
\(\|R_{j_n}-R_{j_m}\|_{\infty,\supp(\pi_0)}
<\Delta_{\min}^{\mathcal P}/2\).
At the fixed context \(x\), the pointwise comparison used in
Lemma~\ref{lem:np_local_stability_selector} then shows that these
rewards all have the same maximizing action, say \(a_\star\): for any
other supported action, its reward disadvantage under \(R_{j_n}\)
can decrease by at most twice that norm when the reward changes to
\(R_{j_m}\), and therefore remains strictly positive.
For every \(a\in S_x\setminus\{a_\star\}\) and \(n\ge n_0\),
\[
R_{j_n}(x,a_\star)-R_{j_n}(x,a)\ge\Delta_{\min}^{\mathcal P}.
\]
Taking the limit in this inequality gives
\(R(x,a_\star)-R(x,a)\ge\Delta_{\min}^{\mathcal P}\).
This holds for every such \(a\), so \(a_\star\) is the unique maximizer
of \(R\) with the same margin. Since \(R\) was arbitrary, the common
set has all the asserted properties.
\end{proof}

\subsection{Proofs of structural lemmas and offline/online results}

\begin{proof}[Proof of Lemma~\ref{lem:np_zero_loss_implies_zero_regret}]
Fix \(p\in\mathcal P\) and \(R\in\mathcal F_{\mathcal P}\), and define
\[
K_{p,R}(x,\mathbf a)
\coloneqq\KL\!\left(P_{R_p}(\cdot\mid x,\mathbf a)
                 \,\middle\|\,P_R(\cdot\mid x,\mathbf a)\right).
\]
Let \(\nu_0\) be the reference context--slate law
\[
\nu_0(dx,d\mathbf a)
=d_0(dx)\,\pi_0(\cdot\mid x)^{\otimes K}(d\mathbf a).
\]
Conditioning first on \((X,\mathbf A)\) and using
Lemma~\ref{lem:excess_is_kl_support} gives
\begin{equation}
\mathcal L_p(R)-\mathcal L_p(R_p)
=\int K_{p,R}(x,\mathbf a)\,\nu_0(dx,d\mathbf a).
\label{eq:reference_excess_kl_identity}
\end{equation}
The integrand is measurable, nonnegative, and bounded by
\(\log K+2B_{\mathcal P}\): its conditional excess-loss representation
is at most the conditional candidate loss, which has this upper bound.
Thus all the expectations in the identity are finite.

Suppose \(\mathcal L_p(R)=\mathcal L_p(R_p)\). For every integer
\(m\ge1\), nonnegativity gives
\[
0=\int K_{p,R}\,d\nu_0
\ge\frac1m\,\nu_0\{K_{p,R}\ge1/m\}.
\]
Hence each set on the right has measure zero. Their countable union is
\(\{K_{p,R}>0\}\), so \(K_{p,R}=0\) \(\nu_0\)-almost surely.
Integrating its indicator first over the slate distribution shows that,
for \(d_0\)-almost every \(x\), \(K_{p,R}(x,\mathbf a)=0\) for
\(\pi_0(\cdot\mid x)^{\otimes K}\)-almost every slate. The equality
case in Lemma~\ref{lem:excess_is_kl_support} then gives equality of the
two choice laws on those slates.

Fix such a context at which both rewards are centered and have
continuous action sections. Write \(\mu=\pi_0(\cdot\mid x)\) and
\(h(a)=R(x,a)-R_p(x,a)\). Since \(K\ge2\), equality of the first and
second choice odds implies
\begin{align*}
e^{R(x,a_1)-R(x,a_2)}
&=\frac{P_R(1\mid x,\mathbf a)}{P_R(2\mid x,\mathbf a)}\\
&=\frac{P_{R_p}(1\mid x,\mathbf a)}{P_{R_p}(2\mid x,\mathbf a)}
=e^{R_p(x,a_1)-R_p(x,a_2)}.
\end{align*}
Taking logarithms yields \(h(a_1)=h(a_2)\) for
\(\mu^{\otimes K}\)-almost every slate. The set on which this equality
fails depends only on \((a_1,a_2)\). Its \(K\)-fold product measure is
therefore its two-fold product measure, because integrating the
remaining \(K-2\) coordinates contributes a factor of one. Consequently,
if \(A,B\) are independent with law \(\mu\), then \(h(A)=h(B)\)
almost surely.

The boundedness of the rewards permits the following second-moment
calculation:
\begin{align*}
0&=\EE[(h(A)-h(B))^2]\\
 &=\EE[h(A)^2]+\EE[h(B)^2]-2\EE[h(A)h(B)]\\
 &=2\EE[h(A)^2]-2\bigl(\EE[h(A)]\bigr)^2
 =2\Var(h(A)).
\end{align*}
Independence justifies the factorization of the cross term. Centering
of both rewards gives \(\EE[h(A)]=0\), so \(\EE[h(A)^2]=0\) and
\(h(A)=0\) almost surely. Lemma~\ref{lem:support_upgrade}, applied to
the two continuous action sections, now yields
\(R(x,a)=R_p(x,a)\) for every \(a\in S_x\).

This conclusion holds for \(d_0\)-almost every context. At each such
context, a maximizer of \(R\) is a maximizer of \(R_p\), so
\(R_p(x,a_R(x))=R_p(x,a_p(x))\). Integrating this zero pointwise regret
proves \(\mathcal G_p(R)=0\). This last conclusion does not require
uniqueness of the maximizer.
\end{proof}

\begin{lemma}[Expected one-step regret versus disagreement mass]
\label{lem:np_regret_vs_disagreement}
Assume \ref{as:NP2}. Then for every \(p,q\in\mathcal P\),
\[
\Delta_{\min}^{\mathcal P}\,
 d_0\{a_q\neq a_p\}
\le
\mathcal G_p(R_q)
\le
\Delta_{\max}^{\mathcal P}\,
 d_0\{a_q\neq a_p\}.
\]
\end{lemma}

\begin{proof}
Let \(D_{p,q}=\{x:a_q(x)\ne a_p(x)\}\), a measurable set by the
selector convention. Outside \(D_{p,q}\), the selected actions coincide
and their reward difference is zero. At every context in \(D_{p,q}\)
where Condition~\ref{as:NP2} holds, \(a_q(x)\) is a supported action
different from the unique maximizer \(a_p(x)\). Hence
\[
\Delta_{\min}^{\mathcal P}
\le R_p(x,a_p(x))-R_p(x,a_q(x))
\le\Delta_{\max}^{\mathcal P}.
\]
Combining the two cases gives the almost-sure inequalities
\[
\Delta_{\min}^{\mathcal P}\mathbf1_{D_{p,q}}(X)
\le R_p(X,a_p(X))-R_p(X,a_q(X))
\le\Delta_{\max}^{\mathcal P}\mathbf1_{D_{p,q}}(X).
\]
The expectation of the middle expression is \(\mathcal G_p(R_q)\),
and \(\EE[\mathbf1_{D_{p,q}}(X)]=d_0(D_{p,q})\). Taking expectations
therefore proves both bounds.
\end{proof}

\begin{lemma}[Local stability of the temperature-zero selector]
\label{lem:np_local_stability_selector}
Assume \ref{as:NP2}. If \(p,q\in\mathcal P\) satisfy
\[
\|R_q-R_p\|_{\infty,\supp(\pi_0)}<\frac{\Delta_{\min}^{\mathcal P}}{2},
\]
then \(a_q(X)=a_p(X)\) \(d_0\)-a.s.
\end{lemma}

\begin{proof}
Put \(d=\|R_q-R_p\|_{\infty,\supp(\pi_0)}\), so
\(2d<\Delta_{\min}^{\mathcal P}\). Fix a context at which the gap
condition for \(R_p\) holds. For any \(a\in S_x\) different from
\(a_p(x)\),
\begin{align*}
R_q(x,a_p(x))-R_q(x,a)
&=R_p(x,a_p(x))-R_p(x,a)\\
&\quad +(R_q-R_p)(x,a_p(x))-(R_q-R_p)(x,a)\\
&\ge\Delta_{\min}^{\mathcal P}
 -|(R_q-R_p)(x,a_p(x))|-|(R_q-R_p)(x,a)|\\
&\ge\Delta_{\min}^{\mathcal P}-2d>0.
\end{align*}
Thus \(a_p(x)\) is also the unique maximizer of \(R_q(x,\cdot)\),
and \(a_q(x)=a_p(x)\). The set of contexts used in this argument has
\(d_0\)-measure one, which proves the assertion.
\end{proof}

\begin{proof}[Proof of Lemma~\ref{prop:np_automatic_positive_regret_isolation}]
Say that two rewards belong to the same selector class when their
selectors agree \(d_0\)-almost surely. Almost-sure equality is reflexive,
symmetric, and transitive, so this is an equivalence relation.
Let \(r=\Delta_{\min}^{\mathcal P}/8\). Compactness of
\(\mathcal F_{\mathcal P}\) gives a finite \(r\)-net with centers
\(R_{p_1},\ldots,R_{p_M}\) in the class. For any
\(R_p\in\mathcal F_{\mathcal P}\), some center satisfies
\[
\|R_p-R_{p_i}\|_{\infty,\supp(\pi_0)}\le r
<\Delta_{\min}^{\mathcal P}/2.
\]
Lemma~\ref{lem:np_local_stability_selector} implies that \(R_p\) and
\(R_{p_i}\) have the same selector class. There are therefore at most
\(M\) distinct selector classes.

Choose one representative selector \(a^{(1)},\ldots,a^{(m)}\) from each
of the distinct classes. If \(m=1\), every pair of rewards has the same
selector almost surely, so all the regrets in the lemma are zero and
we may set \(\varepsilon_{\mathrm{iso}}^{\mathcal P}=1\).
Suppose now that \(m\ge2\). For \(i\ne j\), define
\[
\delta_{ij}\coloneqq d_0\{x:a^{(i)}(x)\ne a^{(j)}(x)\}.
\]
The definition of distinct equivalence classes gives \(\delta_{ij}>0\).
There are finitely many pairs, and hence
\[
\delta_{\mathcal P}\coloneqq\min_{1\le i<j\le m}\delta_{ij}>0.
\]
Changing a representative on a null set does not change any of these
disagreement probabilities.

Set \(\varepsilon_{\mathrm{iso}}^{\mathcal P}
=\Delta_{\min}^{\mathcal P}\delta_{\mathcal P}\).
For \(p,q\in\mathcal P\) in the same selector class, the regret is zero.
For \(p,q\) in distinct classes, their disagreement mass is at least
\(\delta_{\mathcal P}\), so Lemma~\ref{lem:np_regret_vs_disagreement}
gives
\[
\mathcal G_p(R_q)
\ge\Delta_{\min}^{\mathcal P}d_0\{a_p\ne a_q\}
\ge\Delta_{\min}^{\mathcal P}\delta_{\mathcal P}
=\varepsilon_{\mathrm{iso}}^{\mathcal P}.
\]
These two cases establish the claimed isolation.
\end{proof}

\begin{lemma}[Continuity of the truth-centered regret and risk maps]
\label{lem:np_continuity_maps}
Assume \ref{as:NP1}--\ref{as:NP2}. Then the maps
\[
(R^\star,R)\longmapsto \mathcal G_{R^\star}(R)
\qquad\text{and}\qquad
(R^\star,R)\longmapsto \mathcal L_{R^\star}(R)-\mathcal L_{R^\star}(R^\star)
\]
are continuous on \(\mathcal F_{\mathcal P}^2\).
\end{lemma}

\begin{proof}
For a reward \(U\in\mathcal F_{\mathcal P}\), let \(a_U\) denote its
selector, and use \(\mathcal G_U\) and \(\mathcal L_U\) for the regret
and population loss with \(U\) as the truth. Consider a sequence
\((U_n,V_n)\to(U,V)\) in
\(\mathcal F_{\mathcal P}\times\mathcal F_{\mathcal P}\).
For all sufficiently large \(n\), both
\[
\|U_n-U\|_{\infty,\supp(\pi_0)}<\Delta_{\min}^{\mathcal P}/2,
\qquad
\|V_n-V\|_{\infty,\supp(\pi_0)}<\Delta_{\min}^{\mathcal P}/2.
\]
Lemma~\ref{lem:np_local_stability_selector} then gives
\(a_{U_n}=a_U\) and \(a_{V_n}=a_V\) almost surely. Taking a countable
intersection over these \(n\) supplies a common full-measure set of
contexts. On that set,
\begin{align*}
&U_n(x,a_{U_n}(x))-U_n(x,a_{V_n}(x))
       -U(x,a_U(x))+U(x,a_V(x))\\
&\hspace{2em}=(U_n-U)(x,a_U(x))-(U_n-U)(x,a_V(x)).
\end{align*}
It follows that
\[
|\mathcal G_{U_n}(V_n)-\mathcal G_U(V)|
\le2\|U_n-U\|_{\infty,\supp(\pi_0)}\longrightarrow0.
\]
This proves continuity of the regret map.

For the excess-risk map, use the fixed reference context--slate law
\(\nu_0(dx,d\mathbf a)=d_0(dx)\pi_0(\cdot\mid x)^{\otimes K}(d\mathbf a)\).
By Lemma~\ref{lem:excess_is_kl_support},
\[
\mathcal L_{U_n}(V_n)-\mathcal L_{U_n}(U_n)
=\int \KL\!\left(P_{U_n}(\cdot\mid x,\mathbf a)
                 \,\middle\|\,P_{V_n}(\cdot\mid x,\mathbf a)\right)
\,\nu_0(dx,d\mathbf a).
\]
For every reference-supported slate, the two reward vectors converge
coordinatewise to those associated with \(U\) and \(V\). Softmax is
continuous and has strictly positive coordinates. Therefore the finite
sum defining the displayed KL integrand converges to the corresponding
integrand for \((U,V)\). Each integrand lies between zero and
\(\ell_{\max}=\log K+2B_{\mathcal P}\), by its conditional excess-loss
representation and Lemma~\ref{lem:mnl_loss_basic}. The constant function
\(\ell_{\max}\) is integrable under the probability measure \(\nu_0\).
The dominated convergence theorem now gives
\[
\mathcal L_{U_n}(V_n)-\mathcal L_{U_n}(U_n)
\longrightarrow\mathcal L_U(V)-\mathcal L_U(U).
\]
Both domains are metric spaces, so the sequential continuity just proved
is equivalent to continuity.
\end{proof}

\begin{proof}[Proof of Lemma~\ref{lem:np_automatic_fixed_scale_gap}]
Define
\[
\mathcal K_{\varepsilon_0}
\coloneqq\{(U,V)\in\mathcal F_{\mathcal P}^2:
                       \mathcal G_U(V)\ge\varepsilon_0\}.
\]
By Lemma~\ref{lem:np_continuity_maps}, this set is the inverse image of
the closed interval \([\varepsilon_0,\infty)\) under a continuous map.
It is therefore closed in \(\mathcal F_{\mathcal P}^2\).
The latter space is compact as the product of two compact spaces, so
\(\mathcal K_{\varepsilon_0}\) is compact.

If \(\mathcal K_{\varepsilon_0}\) is empty, there are no pairs satisfying
the antecedent of the lemma, and we may take
\(\gamma_{\mathcal P,\varepsilon_0}=1\).
Otherwise, consider the continuous, nonnegative function
\[
\Psi(U,V)\coloneqq\mathcal L_U(V)-\mathcal L_U(U).
\]
Its continuity follows from Lemma~\ref{lem:np_continuity_maps}, and its
nonnegativity follows from Lemma~\ref{lem:excess_is_kl_support}.
Because \(\mathcal K_{\varepsilon_0}\) is nonempty and compact,
\(\Psi\) attains its minimum at some
\((U_0,V_0)\in\mathcal K_{\varepsilon_0}\). If this minimum were zero,
Lemma~\ref{lem:np_zero_loss_implies_zero_regret} would give
\(\mathcal G_{U_0}(V_0)=0\). Membership in
\(\mathcal K_{\varepsilon_0}\), however, gives
\(\mathcal G_{U_0}(V_0)\ge\varepsilon_0>0\), a contradiction.
Thus the attained minimum is strictly positive.

Take
\[
\gamma_{\mathcal P,\varepsilon_0}
\coloneqq\min\!\left\{1,
            \min_{(U,V)\in\mathcal K_{\varepsilon_0}}\Psi(U,V)\right\}.
\]
This is a finite, strictly positive number, and every pair with
\(\mathcal G_U(V)\ge\varepsilon_0\) has
\(\Psi(U,V)\ge\gamma_{\mathcal P,\varepsilon_0}\).
Substituting \(U=R_p\) and \(V=R_q\) proves the claim uniformly over
\(p,q\). No continuity assumption on the indexing map \(p\mapsto R_p\)
is needed: the compactness argument takes place directly in the reward
class.
\end{proof}

\paragraph{Predictable risks and measurability.}
For the concentration arguments, \(\mathscr H_s\) is the history through
round \(s\), including any randomization used by the learner before the
next observation. The sampling rule is measurable with respect to the
preceding history. For a fixed candidate reward, write
\[
\ell_s(R)\coloneqq\ell(\mathbf v_R(X_s,\mathbf A_s),Y_s),
\qquad
m_s(R)\coloneqq\EE[\ell_s(R)\mid\mathscr H_{s-1}].
\]
We define \(m_s(R)\), simultaneously for all \(R\), by integration
against the specified conditional observation kernel. This fixes
versions for which the Lipschitz inequalities below hold simultaneously
on the reward class. In particular, evaluating the random function
\(m_s(\cdot)\) at a fitted reward does not mean that the fitted reward
is held fixed inside an ordinary conditional expectation.

\begin{lemma}[Deviation tail bound via an \(\epsilon\)-net]
\label{lem:bt_tail_support}
Assume \(\mathcal F\) is separable and bounded in
\(\|\cdot\|_{\infty,\supp(\pi_0)}\) by \(B\), and that at each round every slate coordinate is sampled from a distribution absolutely continuous with respect to \(\pi_0(\cdot\mid x)\). Let
\[
\mathcal L_t(R)=\frac1t\sum_{s=1}^t \EE[\ell_s(R)\mid \mathscr H_{s-1}],
\qquad
b_t=\sup_{R\in\mathcal F}|\widehat{\mathcal L}_t(R)-\mathcal L_t(R)|.
\]
If \(\mathcal C_\epsilon\subseteq\mathcal F\) is a finite
\(\epsilon\)-net of \(\mathcal F\), with centers in the class, then
for every \(u>0\),
\[
\PP(b_t\ge u+4\epsilon)
\le
2|\mathcal C_\epsilon|\exp\!\left(-\frac{tu^2}{2(\log K+2B)^2}\right).
\]
\end{lemma}

\begin{proof}
Set \(L=\log K+2B>0\), fix a deterministic \(R\in\mathcal F\), and
let \(Z_s(R)=\ell_s(R)-m_s(R)\). Absolute continuity of the slate
coordinates and full reference measure of the support imply that
\(0\le\ell_s(R)\le L\) almost surely. Conditional expectation
preserves these bounds, so \(0\le m_s(R)\le L\),
\(|Z_s(R)|\le L\), and
\(\EE[Z_s(R)\mid\mathscr H_{s-1}]=0\).

We derive the martingale tail bound explicitly. For
\(z\in[-L,L]\) and \(\lambda\in\R\), convexity of the exponential
function gives the chord bound
\[
e^{\lambda z}
\le\frac{L+z}{2L}e^{\lambda L}
  +\frac{L-z}{2L}e^{-\lambda L}.
\]
Taking conditional expectations with \(z=Z_s(R)\) yields
\[
\EE[e^{\lambda Z_s(R)}\mid\mathscr H_{s-1}]
\le\frac{e^{\lambda L}+e^{-\lambda L}}2
=\cosh(\lambda L)\le e^{\lambda^2L^2/2}.
\]
The last inequality follows from the power series: for every integer
\(j\ge0\), \((2j)!\ge2^j j!\), and hence
\(\sum_{j\ge0}v^{2j}/(2j)!\le\sum_{j\ge0}(v^2/2)^j/j!\).
Since \(\sum_{s=1}^{t-1}Z_s(R)\) is \(\mathscr H_{t-1}\)-measurable,
the tower property gives
\begin{align*}
\EE\exp\!\left(\lambda\sum_{s=1}^tZ_s(R)\right)
&=\EE\!\left[
 e^{\lambda\sum_{s=1}^{t-1}Z_s(R)}
 \EE[e^{\lambda Z_t(R)}\mid\mathscr H_{t-1}]\right]\\
&\le e^{\lambda^2L^2/2}
 \EE\exp\!\left(\lambda\sum_{s=1}^{t-1}Z_s(R)\right).
\end{align*}
Iterating from \(t\) down to zero gives
\(\EE e^{\lambda\sum_{s=1}^tZ_s(R)}\le e^{t\lambda^2L^2/2}\).
For \(u>0\) and \(\lambda>0\), Markov's inequality therefore implies
\[
\PP\!\left(\frac1t\sum_{s=1}^tZ_s(R)\ge u\right)
\le\exp\!\left(-\lambda tu+\frac{t\lambda^2L^2}{2}\right).
\]
Choosing \(\lambda=u/L^2\) gives \(e^{-tu^2/(2L^2)}\).
Applying the same argument to \(-Z_s(R)\), and adding the two one-sided
bounds, gives
\begin{equation}
\PP\!\left(|\widehat{\mathcal L}_t(R)-\mathcal L_t(R)|\ge u\right)
\le2\exp\!\left(-\frac{tu^2}{2L^2}\right).
\label{eq:fixed_reward_martingale_tail}
\end{equation}
No independence across rounds was used.

For the passage from the net to the class, let \(R\in\mathcal F\)
and choose \(R^\circ\in\mathcal C_\epsilon\) with
\(\|R-R^\circ\|_{\infty,\supp(\pi_0)}\le\epsilon\).
Lemma~\ref{lem:mnl_loss_basic} gives
\(|\ell_s(R)-\ell_s(R^\circ)|\le2\epsilon\) on supported
observations. Integration against the conditional observation kernel
also gives \(|m_s(R)-m_s(R^\circ)|\le2\epsilon\).
Averaging these inequalities separately yields
\[
|\widehat{\mathcal L}_t(R)-\widehat{\mathcal L}_t(R^\circ)|\le2\epsilon,
\qquad
|\mathcal L_t(R)-\mathcal L_t(R^\circ)|\le2\epsilon.
\]
The triangle inequality then gives
\[
|\widehat{\mathcal L}_t(R)-\mathcal L_t(R)|
\le|\widehat{\mathcal L}_t(R^\circ)-\mathcal L_t(R^\circ)|+4\epsilon.
\]
Taking the supremum over \(R\) proves
\begin{equation}
b_t\le\max_{R^\circ\in\mathcal C_\epsilon}
 |\widehat{\mathcal L}_t(R^\circ)-\mathcal L_t(R^\circ)|+4\epsilon.
\label{eq:uniform_deviation_net_extension}
\end{equation}
Both risk functions are Lipschitz in \(R\). Thus, if \(D\) is a
countable dense subset of the separable class, the supremum defining
\(b_t\) equals the supremum over \(D\): approximate any \(R\) by a
sequence in \(D\) and use continuity. Consequently \(b_t\) is measurable.

By \eqref{eq:uniform_deviation_net_extension}, the event
\(\{b_t\ge u+4\epsilon\}\) is contained in the union, over
\(R^\circ\in\mathcal C_\epsilon\), of the events
\(\{|\widehat{\mathcal L}_t(R^\circ)-\mathcal L_t(R^\circ)|\ge u\}\).
The union bound and \eqref{eq:fixed_reward_martingale_tail} therefore give
\[
\PP(b_t\ge u+4\epsilon)
\le2|\mathcal C_\epsilon|\exp\!\left(-\frac{tu^2}{2L^2}\right),
\]
as asserted.
\end{proof}

\subsubsection{Offline exact ERM fixed-scale control}

\begin{lemma}[Offline exact ERM: exponential fixed-scale control]
\label{lem:np_offline_fixed_scale_control}
\label{thm:np_offline_fixed_scale_control}
Assume \ref{as:NP1}--\ref{as:NP2}. Fix any \(\varepsilon_0>0\), and let \(\gamma_{\mathcal P,\varepsilon_0}>0\) be the fixed-scale loss-gap constant from Lemma~\ref{lem:np_automatic_fixed_scale_gap}. Consider offline data generated under truth \(p\in\mathcal P\) by
\[
X_s\sim d_0,
\qquad
\mathbf A_s\sim \pi_0(\cdot\mid X_s)^{\otimes K},
\qquad
Y_s\sim P_{R_p}(\cdot\mid X_s,\mathbf A_s),
\qquad s=1,\ldots,n,
\]
independently across \(s\). Let
\[
\widehat R_n^{\mathrm{off}}
\in
\argmin_{R\in\mathcal F_{\mathcal P}}
\frac1n\sum_{s=1}^n
\ell\!\big(\mathbf v_R(X_s,\mathbf A_s),Y_s\big)
\]
be an exact offline ERM. Define
\[
\ell_{\max}\coloneqq \log K+2B_{\mathcal P},
\qquad
N_{\varepsilon_0}^{\mathrm{off}}
\coloneqq
\mathcal N\!\left(
\mathcal F_{\mathcal P},
\gamma_{\mathcal P,\varepsilon_0}/32,
\|\cdot\|_{\infty,\supp(\pi_0)}
\right),
\]
and
\[
c_{\varepsilon_0}^{\mathrm{off}}
\coloneqq
\frac{\gamma_{\mathcal P,\varepsilon_0}^2}{128(\log K+2B_{\mathcal P})^2}.
\]
Then, for every \(n\ge1\),
\[
\sup_{p\in\mathcal P}
\PP_p\!\left(
\mathcal G_p(\widehat R_n^{\mathrm{off}})\ge \varepsilon_0
\right)
\le
2N_{\varepsilon_0}^{\mathrm{off}}
\exp\{-c_{\varepsilon_0}^{\mathrm{off}}n\}.
\]
\end{lemma}

\begin{proof}
Fix \(p\in\mathcal P\), and abbreviate
\(\gamma_0=\gamma_{\mathcal P,\varepsilon_0}\) and
\(L=\log K+2B_{\mathcal P}\). Let
\[
\widehat L_n(R)=\frac1n\sum_{s=1}^n\ell_s(R),
\qquad
b_n=\sup_{R\in\mathcal F_{\mathcal P}}
          |\widehat L_n(R)-\mathcal L_p(R)|.
\]
The offline observations are independent and each has the reference
context--slate distribution and truth choice law used to define
\(\mathcal L_p\). Thus, for every fixed \(R\), its predictable loss is
\(m_s(R)=\mathcal L_p(R)\) at every round. This equality holds
simultaneously in \(R\) when both sides are defined by integration
against that common observation kernel. Consequently \(b_n\) is
exactly the uniform deviation appearing in Lemma~\ref{lem:bt_tail_support}.
The class is compact, hence separable, and has finite covering numbers.
The hypotheses of that lemma are therefore satisfied.

Because \(R_p\) is a feasible competitor for the exact ERM,
\(\widehat L_n(\widehat R_n^{\mathrm{off}})\le\widehat L_n(R_p)\).
On every sample path where the deviation is defined,
\begin{align}
\mathcal L_p(\widehat R_n^{\mathrm{off}})-\mathcal L_p(R_p)
&=\bigl[\mathcal L_p(\widehat R_n^{\mathrm{off}})
             -\widehat L_n(\widehat R_n^{\mathrm{off}})\bigr]
  +\bigl[\widehat L_n(\widehat R_n^{\mathrm{off}})-\widehat L_n(R_p)\bigr]
  \nonumber\\
&\quad+\bigl[\widehat L_n(R_p)-\mathcal L_p(R_p)\bigr]
 \nonumber\\
&\le b_n+0+b_n=2b_n.
\label{eq:offline_erm_excess_step}
\end{align}
The use of the random fitted reward in this inequality is justified
by the supremum over the entire class in the definition of \(b_n\).

Let \(E_n=\{\mathcal G_p(\widehat R_n^{\mathrm{off}})\ge\varepsilon_0\}\).
On \(E_n\), Lemma~\ref{lem:np_automatic_fixed_scale_gap} gives
\(\mathcal L_p(\widehat R_n^{\mathrm{off}})-\mathcal L_p(R_p)\ge\gamma_0\).
Combining this with \eqref{eq:offline_erm_excess_step} yields
\begin{equation}
E_n\subseteq\{b_n\ge\gamma_0/2\}
    \subseteq\{b_n\ge\gamma_0/4\}.
\label{eq:offline_bad_event_subset}
\end{equation}
The weaker threshold matches the net radius used below.

Choose an in-class net \(\mathcal C\) of radius
\(r=\gamma_0/32\) and cardinality
\(N_{\varepsilon_0}^{\mathrm{off}}\), and set \(u=\gamma_0/8\).
The definition of the covering number permits such a net; all its
centers obey the same reward envelope. Since
\(u+4r=\gamma_0/8+\gamma_0/8=\gamma_0/4\),
the fixed-reward bound \eqref{eq:fixed_reward_martingale_tail} and a
union bound give
\begin{equation}
\PP_p\!\left(\max_{R\in\mathcal C}
 |\widehat L_n(R)-\mathcal L_p(R)|\ge\gamma_0/8\right)
\le2N_{\varepsilon_0}^{\mathrm{off}}
       \exp\!\left(-\frac{n\gamma_0^2}{128L^2}\right).
\label{eq:offline_net_concentration}
\end{equation}
The net extension \eqref{eq:uniform_deviation_net_extension} gives
\(b_n\le\max_{R\in\mathcal C}|\widehat L_n(R)-\mathcal L_p(R)|+4r\).
Thus \(b_n\ge\gamma_0/4\) implies that this maximum is at least
\(\gamma_0/4-4r=\gamma_0/8\), and hence
\begin{align}
\PP_p(b_n\ge\gamma_0/4)
&\le 2N_{\varepsilon_0}^{\mathrm{off}}
       \exp\!\left(-\frac{n(\gamma_0/8)^2}{2L^2}\right)
 \nonumber\\
&=2N_{\varepsilon_0}^{\mathrm{off}}
       \exp\!\left(-\frac{n\gamma_0^2}{128L^2}\right)
 =2N_{\varepsilon_0}^{\mathrm{off}}
       e^{-c_{\varepsilon_0}^{\mathrm{off}}n}.
\label{eq:offline_bn_tail}
\end{align}
By \eqref{eq:offline_bad_event_subset}, the same upper bound applies to
\(\PP_p(E_n)\). The envelope, covering number, and loss-gap constant
were chosen for the whole class, independently of \(p\). Taking the
supremum over \(p\in\mathcal P\) proves the lemma.
\end{proof}

\begin{proof}[Proof of Theorem~\ref{thm:np_offline_log_accuracy}]
Fix \(p\in\mathcal P\) and \(n\ge1\). The fitted reward belongs to
\(\mathcal F_{\mathcal P}\), so Lemma~\ref{prop:np_automatic_positive_regret_isolation}
ensures that its regret is either zero or at least
\(\varepsilon_{\mathrm{iso}}^{\mathcal P}\). Since any supported
selector has regret at most \(\Delta_{\max}^{\mathcal P}\),
\[
\mathcal G_p(\widehat R_n^{\mathrm{off}})
\le\Delta_{\max}^{\mathcal P}\,
 \mathbf1\{\mathcal G_p(\widehat R_n^{\mathrm{off}})
                      \ge\varepsilon_{\mathrm{iso}}^{\mathcal P}\}.
\]
Taking expectations and applying
Lemma~\ref{lem:np_offline_fixed_scale_control} at
\(\varepsilon_0=\varepsilon_{\mathrm{iso}}^{\mathcal P}\) gives
\begin{align*}
\EE_p[\mathcal G_p(\widehat R_n^{\mathrm{off}})]
&\le\Delta_{\max}^{\mathcal P}
 \PP_p\!\left(\mathcal G_p(\widehat R_n^{\mathrm{off}})
                  \ge\varepsilon_{\mathrm{iso}}^{\mathcal P}\right)\\
&\le2\Delta_{\max}^{\mathcal P}N_0^{\mathrm{off}}
                  e^{-c_0^{\mathrm{off}}n}.
\end{align*}
The constants are uniform in \(p\), which permits taking the supremum
over \(\mathcal P\).

If \(\Delta_{\max}^{\mathcal P}=0\), the regret is identically zero,
so no logarithm of this constant is needed. Otherwise set
\(A=2\Delta_{\max}^{\mathcal P}N_0^{\mathrm{off}}>0\) and
\(c=c_0^{\mathrm{off}}>0\). For \(\eps>0\), monotonicity of the
logarithm and positivity of \(c\) give the equivalences
\[
Ae^{-cn}\le\eps
\quad\Longleftrightarrow\quad
\log A-cn\le\log\eps
\quad\Longleftrightarrow\quad
n\ge c^{-1}\log(A/\eps).
\]
Rounding this last threshold upward and imposing \(n\ge1\) yields the
sample-size condition in the statement. For the fixed class,
\(c^{-1}\log(A/\eps)=c^{-1}\log(1/\eps)+c^{-1}\log A\), so the
required sample size has the asserted logarithmic order.
\end{proof}

\begin{lemma}[Fixed-scale control from a truth-centered loss gap]
\label{thm:np_online_fixed_scale_control}
Assume \ref{as:NP1}--\ref{as:NP2}, and fix \(\varepsilon_0>0\). Let \(\gamma_{\mathcal P,\varepsilon_0}>0\) be any witness from Lemma~\ref{lem:np_automatic_fixed_scale_gap}. Define
\[
\gamma=\beta^{-K}\gamma_{\mathcal P,\varepsilon_0},
\qquad
N_\gamma=\mathcal N\!\big(\mathcal F_{\mathcal P},\gamma/32,\|\cdot\|_{\infty,\supp(\pi_0)}\big),
\qquad
c_\gamma=\frac{\gamma^2}{128(\log K+2B_{\mathcal P})^2}.
\]
Then, for every truth \(p\in\mathcal P\) and every \(t\ge1\),
\[
\PP_p(\mathcal G_p(\widehat R_t)\ge\varepsilon_0)
\le
2N_\gamma e^{-c_\gamma t}.
\]
\end{lemma}

\begin{proof}
Fix a truth \(p\), and let \(\mathscr H_s\) and \(m_s(R)\) have the
predictable-risk meanings specified above. Put
\[
\widehat L_t(R)=\frac1t\sum_{s=1}^t\ell_s(R),
\qquad
\bar L_t(R)=\frac1t\sum_{s=1}^t m_s(R),
\qquad
b_t=\sup_{R\in\mathcal F_{\mathcal P}}
                  |\widehat L_t(R)-\bar L_t(R)|.
\]
The function \(\bar L_t\) is generally random, because the sampling
policies depend on the past. It is not identified with the reference
population loss \(\mathcal L_p\).

At round \(s=1\), the sampling policy is \(\pi_0\), so its density
relative to \(\pi_0\) is one. At every round \(s\ge2\), the policy is
\(\pi_{\widehat R_{s-1}}\), and the fitted reward belongs to
\(\mathcal F_{\mathcal P}\). Lemma~\ref{lem:greedy_kl_tilt_lr} therefore
gives, conditionally on the preceding history,
\[
\frac{d\pi_s(\cdot\mid x)}{d\pi_0(\cdot\mid x)}(a)
\ge\beta^{-1},
\qquad \beta=e^{2\eta B_{\mathcal P}}\ge1,
\]
for reference-almost every action and for \(d_0\)-almost every context.
This also holds at \(s=1\).

For an arbitrary candidate \(R\in\mathcal F_{\mathcal P}\), define
\(K_{p,R}\) as in \eqref{eq:reference_excess_kl_identity}. Integrating
first over the preference label, whose conditional law is \(P_{R_p}\),
and then over the context and slate gives
\begin{align}
m_s(R)-m_s(R_p)
&=\int_{\cX}\int_{\cA^K}K_{p,R}(x,\mathbf a)
                 \,\pi_s(\cdot\mid x)^{\otimes K}(d\mathbf a)\,d_0(dx)
 \nonumber\\
&\ge\beta^{-K}\int_{\cX}\int_{\cA^K}K_{p,R}(x,\mathbf a)
                 \,\pi_0(\cdot\mid x)^{\otimes K}(d\mathbf a)\,d_0(dx)
 \nonumber\\
&=\beta^{-K}\bigl(\mathcal L_p(R)-\mathcal L_p(R_p)\bigr).
\label{eq:online_conditional_excess_transfer}
\end{align}
The inequality is Lemma~\ref{lem:slate_domination_betaK}, applied to
the nonnegative function \(K_{p,R}\). The last equality is
\eqref{eq:reference_excess_kl_identity}. The conditional risks here
are explicit kernel integrals, so the displayed comparison holds
simultaneously for all candidates at each admissible history.
Averaging \eqref{eq:online_conditional_excess_transfer} over
\(s=1,\ldots,t\) yields
\begin{equation}
\bar L_t(R)-\bar L_t(R_p)
\ge\beta^{-K}\bigl(\mathcal L_p(R)-\mathcal L_p(R_p)\bigr)
\qquad(R\in\mathcal F_{\mathcal P}).
\label{eq:online_predictable_excess_transfer}
\end{equation}

Since the inequality is uniform in the candidate, it can now be
evaluated at the data-dependent \(\widehat R_t\). On the event
\(E_t=\{\mathcal G_p(\widehat R_t)\ge\varepsilon_0\}\),
Lemma~\ref{lem:np_automatic_fixed_scale_gap} and
\eqref{eq:online_predictable_excess_transfer} imply
\[
\bar L_t(\widehat R_t)-\bar L_t(R_p)
\ge\beta^{-K}\gamma_{\mathcal P,\varepsilon_0}=\gamma.
\]
On the other hand, exact empirical minimization gives
\(\widehat L_t(\widehat R_t)\le\widehat L_t(R_p)\). Therefore, on
every sample path,
\begin{align*}
\bar L_t(\widehat R_t)-\bar L_t(R_p)
&=\bigl[\bar L_t(\widehat R_t)-\widehat L_t(\widehat R_t)\bigr]
 +\bigl[\widehat L_t(\widehat R_t)-\widehat L_t(R_p)\bigr]\\
&\quad+\bigl[\widehat L_t(R_p)-\bar L_t(R_p)\bigr]\\
&\le b_t+0+b_t=2b_t.
\end{align*}
It follows that
\begin{equation}
E_t\subseteq\{b_t\ge\gamma/2\}
    \subseteq\{b_t\ge\gamma/4\}.
\label{eq:online_bad_event_subset}
\end{equation}
The substitution is justified by uniformity in the reward argument;
no predictability of \(\widehat R_t\) is required.

The reward class is compact and hence separable. Its envelope is
\(B_{\mathcal P}\), and every online slate coordinate is absolutely
continuous with respect to the reference policy. Choose an in-class
net of radius \(r=\gamma/32\) and size \(N_\gamma\), and put
\(u=\gamma/8\). All constants are independent of the chosen truth.
Since \(u+4r=\gamma/4\), \eqref{eq:online_bad_event_subset} and
Lemma~\ref{lem:bt_tail_support} give
\begin{align*}
\PP_p(E_t)
&\le\PP_p(b_t\ge\gamma/4)\\
&\le2N_\gamma
 \exp\!\left(-\frac{t(\gamma/8)^2}{2(\log K+2B_{\mathcal P})^2}\right)\\
&=2N_\gamma
 \exp\!\left(-\frac{t\gamma^2}{128(\log K+2B_{\mathcal P})^2}\right)
 =2N_\gamma e^{-c_\gamma t}.\qedhere
\end{align*}
\end{proof}

\begin{proof}[Proof of Theorem~\ref{thm:np_bounded_regret}]
Fix \(p\in\mathcal P\). The initialization \(\widehat R_0=0\) need
not belong to \(\mathcal F_{\mathcal P}\), but its selector is supported,
so \(0\le\mathcal G_p(\widehat R_0)\le\Delta_{\max}^{\mathcal P}\).
For every \(t\ge1\), the exact ERM does belong to
\(\mathcal F_{\mathcal P}\). Lemma~\ref{prop:np_automatic_positive_regret_isolation}
and the definition \(\varepsilon_0=\varepsilon_{\mathrm{iso}}^{\mathcal P}\)
therefore give, pathwise,
\[
0\le\mathcal G_p(\widehat R_t)
\le\Delta_{\max}^{\mathcal P}
       \mathbf1\{\mathcal G_p(\widehat R_t)\ge\varepsilon_0\}.
\]
The fixed-scale estimate in
Lemma~\ref{thm:np_online_fixed_scale_control} implies
\[
\EE_p[\mathcal G_p(\widehat R_t)]
\le\Delta_{\max}^{\mathcal P}
  \min\{1,2N_\gamma e^{-c_\gamma t}\}.
\]
The minimum with one is valid because the intervening quantity is a
probability. Consequently, for every \(T\ge1\),
\begin{equation}
\Regret_{0,p}^{\mathrm{ERM}}(T)
\le\Delta_{\max}^{\mathcal P}
 \left(1+\sum_{t=1}^{T-1}\min\{1,2N_\gamma e^{-c_\gamma t}\}\right).
\label{eq:online_regret_tail_sum}
\end{equation}

Let \(m=\lceil c_\gamma^{-1}\log(2N_\gamma)\rceil\).
The class is nonempty and compact, so \(1\le N_\gamma<\infty\),
and the positive loss gap gives \(c_\gamma>0\). In particular, \(m\)
is finite and \(2N_\gamma e^{-c_\gamma m}\le1\). The summands with
\(1\le t\le m\) are each at most one. For \(t=m+j\), \(j\ge1\),
\[
\min\{1,2N_\gamma e^{-c_\gamma(m+j)}\}
\le 2N_\gamma e^{-c_\gamma m}e^{-c_\gamma j}
\le e^{-c_\gamma j}.
\]
All terms are nonnegative, so extending the finite sum in
\eqref{eq:online_regret_tail_sum} to an infinite sum gives
\begin{align*}
\sum_{t=1}^{T-1}\min\{1,2N_\gamma e^{-c_\gamma t}\}
&\le m+\sum_{j=1}^{\infty}e^{-c_\gamma j}\\
&=m+\frac{e^{-c_\gamma}}{1-e^{-c_\gamma}}
=m+\frac1{e^{c_\gamma}-1}.
\end{align*}
Substitution into \eqref{eq:online_regret_tail_sum} gives the stated
bound. Its right-hand side depends only on the model class, \(K\),
and \(\eta\), not on \(p\) or \(T\). Taking both suprema proves the
theorem.
\end{proof}

\section{Bounded temperature-zero regret under the general preference model}
\label{app:gp_temperature_zero}

The main results use MNL/Bradley--Terry feedback, where the reward
itself defines the selected action. Here we make explicit the
finite-action, finite-class extension based on the Nash score of a
KL-regularized preference game. The game is the general-preference
formulation of \citet{wu2025greedy}; the data-collection rule below
uses two independent draws from the current policy, as in this paper.
Preference probabilities equal to zero or one are allowed.

Let \(\cA\) be finite, let \(\eta>0\), and let
\(\mathfrak P_{\mathrm{GP}}\) be a nonempty finite class of measurable
kernels \(P:\cX\times\cA\times\cA\to[0,1]\) satisfying
\[
P(x,a,b)+P(x,b,a)=1.
\]
The true kernel belongs to this class. Write
\(S_x=\{a\in\cA:\pi_0(a\mid x)>0\}\), and use a fixed ordering of
\(\cA\) to break all selector ties. For distributions \(\mu,\nu\)
on \(S_x\), define
\begin{equation}
J_{P,x}(\mu,\nu)
=\sum_{a,b\in S_x}\mu(a)\nu(b)P(x,a,b)
-\eta^{-1}\KL(\mu\|\pi_0(\cdot\mid x))
+\eta^{-1}\KL(\nu\|\pi_0(\cdot\mid x)).
\label{eq:gp_regularized_game}
\end{equation}
The first argument is maximized and the second is minimized. All the
KL divergences in this finite simplex are finite, with the convention
\(0\log0=0\).

\begin{lemma}[The Nash policy and its reference coverage]
\label{lem:gp_nash_coverage}
For every \(P\in\mathfrak P_{\mathrm{GP}}\), the game
\eqref{eq:gp_regularized_game} has a unique symmetric Nash policy
\(\pi_P\). It is measurable in the context and satisfies
\begin{equation}
\phi_P(x,a)=\sum_{b\in S_x}\pi_P(b\mid x)P(x,a,b),
\qquad
\pi_P(a\mid x)
=\frac{\pi_0(a\mid x)e^{\eta\phi_P(x,a)}}
       {\sum_{c\in S_x}\pi_0(c\mid x)e^{\eta\phi_P(x,c)}}.
\label{eq:gp_nash_tilt}
\end{equation}
For all supported actions,
\[
0\le\phi_P(x,a)\le1,
\qquad
e^{-\eta}\le\frac{\pi_P(a\mid x)}{\pi_0(a\mid x)}\le e^\eta.
\]
\end{lemma}

\begin{proof}
Fix \(x\) and abbreviate \(\rho=\pi_0(\cdot\mid x)\).
For a probability vector \(\nu\) on the full finite action set, put
\[
u_\nu(a)=\sum_{b\in\cA}\nu(b)P(x,a,b),
\qquad
T(\nu)(a)=\frac{\rho(a)e^{\eta u_\nu(a)}}
                   {\sum_c\rho(c)e^{\eta u_\nu(c)}}.
\]
Here \(u_\nu(a)\in[0,1]\). The denominator lies in \([1,e^\eta]\),
so \(T\) is continuous on the finite-dimensional probability simplex
and takes values in that same simplex, with support contained in
\(S_x\). The simplex is nonempty, compact, and convex. Brouwer's
fixed-point theorem therefore supplies a vector \(v=T(v)\).
It is supported on \(S_x\), and in fact \(v(a)>0\) for each
\(a\in S_x\).

By Lemma~\ref{lem:greedy_kl_tilt_lr}, applied with reward \(u_v\),
\(T(v)\) is the unique maximizer of
\(\sum_a\mu(a)u_v(a)-\eta^{-1}\KL(\mu\|\rho)\).
The remaining term in \(J_{P,x}(\mu,v)\) is independent of \(\mu\).
Thus \(v=T(v)\) maximizes \(J_{P,x}(\mu,v)\), and the variational
identity gives
\begin{equation}
J_{P,x}(\mu,v)
=J_{P,x}(v,v)-\eta^{-1}\KL(\mu\|v).
\label{eq:gp_cross_value}
\end{equation}
The kernel symmetry and cancellation of the regularizers imply
\[
J_{P,x}(\mu,\nu)+J_{P,x}(\nu,\mu)
=\sum_{a,b}\mu(a)\nu(b)\bigl(P(x,a,b)+P(x,b,a)\bigr)=1.
\]
In particular, \(J_{P,x}(v,v)=1/2\). It follows that for all
\(\mu,\nu\) on \(S_x\),
\[
J_{P,x}(\mu,v)\le\tfrac12,
\qquad
J_{P,x}(v,\nu)=1-J_{P,x}(\nu,v)\ge\tfrac12.
\]
These are precisely the saddle inequalities, so \((v,v)\) is a
symmetric Nash equilibrium.

If \(w\) is another symmetric Nash policy, it must maximize
\(J_{P,x}(\mu,w)\); uniqueness of the KL-regularized best response
therefore gives \(w=T(w)\). Apply \eqref{eq:gp_cross_value} to the two
fixed points to obtain
\[
J_{P,x}(w,v)=\tfrac12-\eta^{-1}\KL(w\|v),
\qquad
J_{P,x}(v,w)=\tfrac12-\eta^{-1}\KL(v\|w).
\]
Adding these identities and using the symmetry identity above gives
\(\KL(w\|v)+\KL(v\|w)=0\). Nonnegativity and its equality case
imply \(v=w\). Define \(\pi_P(\cdot\mid x)=v\). The fixed-point
formula is \eqref{eq:gp_nash_tilt}. Its numerator factor
\(e^{\eta\phi_P(x,a)}\) and normalizer both lie in \([1,e^\eta]\),
which gives the claimed density bounds.

To check measurability as \(x\) varies, let \((q_j)_{j\ge1}\) be a
countable dense subset of the full action simplex and restore the
notation \(T_x\). For every fixed \(q_j\), the residual
\(r_j(x)=\|T_x(q_j)-q_j\|_2\) is measurable. For each integer \(n\ge1\),
let \(j_n(x)\) be the least \(j\) with \(r_j(x)<1/n\).
Such an index exists by continuity of the residual and existence of
a fixed point. It is measurable since
\[
\{x:j_n(x)=j\}
=\{x:r_j(x)<1/n\}\cap\bigcap_{i<j}\{x:r_i(x)\ge1/n\}.
\]
Thus \(q_{j_n(x)}\) is a measurable vector-valued function of \(x\).
For a fixed context, every convergent subsequence of these vectors
has residual zero at its limit, by continuity of \(T_x\).
The unique fixed point is \(\pi_P(\cdot\mid x)\), so compactness of
the simplex implies that the entire sequence converges to it: a
subsequence staying a positive distance away would have a further
convergent subsequence with a different limit. Hence \(\pi_P\) is the
pointwise limit of measurable vectors and is measurable.
\end{proof}

Define the temperature-zero selector and regret by
\[
a_P(x)\in\argmax_{a\in S_x}\phi_P(x,a),
\qquad
\mathcal G_P^{\mathrm{GP}}(Q)
=\EE_{X\sim d_0}\bigl[\phi_P(X,a_P(X))-
                         \phi_P(X,a_Q(X))\bigr].
\]
These selectors are measurable because the action set is finite, the
scores and support indicators are measurable, and ties are resolved
by a fixed ordering. Since every score lies in \([0,1]\),
\(0\le\mathcal G_P^{\mathrm{GP}}(Q)\le1\).
For \(p,q\in[0,1]\), use the squared Hellinger convention
\[
h^2(p,q)
\coloneqq\tfrac12\left[(\sqrt p-\sqrt q)^2
              +(\sqrt{1-p}-\sqrt{1-q})^2\right]
=1-\sqrt{pq}-\sqrt{(1-p)(1-q)}.
\]
Thus \(0\le h^2\le1\), and \(h^2(p,q)=0\) exactly when \(p=q\).
Define the reference-comparison separation
\begin{equation}
H_0^2(P,Q)
\coloneqq\EE_{X\sim d_0}
 \sum_{a,b\in S_X}\pi_0(a\mid X)\pi_0(b\mid X)
                    h^2(P(X,a,b),Q(X,a,b)).
\label{eq:gp_reference_hellinger}
\end{equation}

\begin{lemma}[Reference identifiability of the Nash selector]
\label{lem:gp_reference_identification}
In the finite-action setting,
\(H_0^2(P,Q)=0\) implies \(\mathcal G_P^{\mathrm{GP}}(Q)=0\).
Consequently, in a finite model class, positive-regret pairs have a
uniform strictly positive reference Hellinger separation whenever
any such pairs exist.
\end{lemma}

\begin{proof}
The integrand in \eqref{eq:gp_reference_hellinger} is nonnegative.
If its expectation is zero, then it is zero for \(d_0\)-almost every
context. At each such context the sum is finite, and the coefficient
\(\pi_0(a\mid x)\pi_0(b\mid x)\) is strictly positive for every
\(a,b\in S_x\). Therefore every summand has
\(h^2(P(x,a,b),Q(x,a,b))=0\), and hence
\(P(x,a,b)=Q(x,a,b)\) for all supported pairs.

The two regularized games at this context are identical on the
reference-supported simplex. Uniqueness in
Lemma~\ref{lem:gp_nash_coverage} gives \(\pi_P=\pi_Q\) there.
Substitution into the score formula shows
\(\phi_P(x,a)=\phi_Q(x,a)\) for every supported action. Thus any
maximizer selected under \(Q\) is also a maximizer under \(P\), and
the pointwise regret is zero. Integrating gives
\(\mathcal G_P^{\mathrm{GP}}(Q)=0\).

Its contrapositive says that every positive-regret pair has
\(H_0^2(P,Q)>0\). There are finitely many ordered pairs in a finite
class, so the minimum of these strictly positive numbers, over a
nonempty set of positive-regret pairs, is strictly positive.
\end{proof}

We now specify the greedy MLE-to-Nash procedure. Choose any fixed
\(\widehat P_0\in\mathfrak P_{\mathrm{GP}}\) for the initial selector,
and initialize the sampling policy at \(\pi_1=\pi_0\). Under truth
\(P\), round \(s\) is generated conditionally on the preceding
history by
\[
X_s\sim d_0,\qquad
(A_s,B_s)\sim\pi_s(\cdot\mid X_s)^{\otimes2},\qquad
Y_s\sim\operatorname{Bernoulli}(P(X_s,A_s,B_s)).
\]
Here \(Y_s=1\) means that the first action is preferred. For a
candidate \(Q\), define
\[
f_Q(y\mid x,a,b)=
\begin{cases}
Q(x,a,b),&y=1,\\
1-Q(x,a,b),&y=0,
\end{cases}
\qquad
\Lambda_t(Q)=\prod_{s=1}^t f_Q(Y_s\mid X_s,A_s,B_s).
\]
Choose a measurable exact maximizer
\[
\widehat P_t\in\argmax_{Q\in\mathfrak P_{\mathrm{GP}}}\Lambda_t(Q),
\qquad
\pi_{t+1}=\pi_{\widehat P_t}.
\]
A fixed ordering of the finite class provides a measurable maximizer,
including in the presence of ties. The product formulation avoids
undefined logarithms when a candidate assigns zero probability to an
observed label. It is equivalent to minimizing the extended-valued
negative log-likelihood.

\begin{theorem}[Finite-class general-preference guarantee]
\label{thm:gp_bounded_temperature_zero}
Let \(M=|\mathfrak P_{\mathrm{GP}}|\), and define
\[
\mathcal B=\{(P,Q)\in\mathfrak P_{\mathrm{GP}}^2:
                         \mathcal G_P^{\mathrm{GP}}(Q)>0\}.
\]
If \(\mathcal B\ne\varnothing\), set
\(h_\star=\min_{(P,Q)\in\mathcal B}H_0^2(P,Q)>0\); if
\(\mathcal B=\varnothing\), set \(h_\star=1\).
Let \(\kappa=e^{-2\eta}h_\star>0\).
Then, for every truth \(P\in\mathfrak P_{\mathrm{GP}}\) and \(t\ge1\),
\[
\PP_P\!\left(\mathcal G_P^{\mathrm{GP}}(\widehat P_t)>0\right)
\le(M-1)e^{-\kappa t},
\qquad
\EE_P[\mathcal G_P^{\mathrm{GP}}(\widehat P_t)]
\le\min\{1,(M-1)e^{-\kappa t}\}.
\]
In particular,
\[
\sup_{P\in\mathfrak P_{\mathrm{GP}}}\sup_{T\ge1}
\sum_{t=0}^{T-1}\EE_P[\mathcal G_P^{\mathrm{GP}}(\widehat P_t)]
\le1+\frac{M-1}{e^\kappa-1}<\infty.
\]
\end{theorem}

\begin{proof}
If \(\mathcal B\) is empty, every candidate has zero regret under every
truth, so all the displayed claims hold. Otherwise
Lemma~\ref{lem:gp_reference_identification} justifies \(h_\star>0\).
Fix a truth \(P\), and consider any candidate \(Q\) for which
\(\mathcal G_P^{\mathrm{GP}}(Q)>0\). Let \(\mathscr H_s\) be the
history through round \(s\), and abbreviate
\(Z_s=(X_s,A_s,B_s)\).

At round one the sampling density relative to the reference is one.
At later rounds, Lemma~\ref{lem:gp_nash_coverage} gives
\(\pi_s(a\mid x)\ge e^{-\eta}\pi_0(a\mid x)\) for every supported
action. Thus, at every round and conditional on every admissible
preceding history, the pair probabilities are at least
\(e^{-2\eta}\) times their reference counterparts. Since \(h^2\) is
nonnegative, integrating this product bound yields
\begin{align}
&\EE_P[h^2(P(Z_s),Q(Z_s))\mid\mathscr H_{s-1}]
 \nonumber\\
&\quad\ge e^{-2\eta}\EE_{X\sim d_0}
 \sum_{a,b\in S_X}\pi_0(a\mid X)\pi_0(b\mid X)
                   h^2(P(X,a,b),Q(X,a,b))
 \nonumber\\
&\quad=e^{-2\eta}H_0^2(P,Q)\ge e^{-2\eta}h_\star=\kappa.
\label{eq:gp_conditional_hellinger_gap}
\end{align}

For each observation, the true conditional probability of its realized
label is positive almost surely: conditionally on \(Z_s\), the
probability of seeing a label assigned probability zero is the sum of
its zero probabilities. A countable union over rounds shows that
\(\Lambda_t(P)>0\) simultaneously for all finite \(t\), almost surely.
On this event define
\[
W_t(Q)=\sqrt{\frac{\Lambda_t(Q)}{\Lambda_t(P)}},
\qquad W_0(Q)=1.
\]
The values on the null complement may be set arbitrarily.
Conditionally on the preceding history and the current pair, direct
summation over the two possible labels gives
\begin{align*}
&\EE_P\!\left[
 \sqrt{\frac{f_Q(Y_s\mid Z_s)}{f_P(Y_s\mid Z_s)}}
 \,\middle|\,\mathscr H_{s-1},Z_s\right]\\
&\quad=\sum_{y:\,f_P(y\mid Z_s)>0}
                  \sqrt{f_P(y\mid Z_s)f_Q(y\mid Z_s)}\\
&\quad=\sqrt{P(Z_s)Q(Z_s)}
       +\sqrt{(1-P(Z_s))(1-Q(Z_s))}
 =1-h^2(P(Z_s),Q(Z_s)).
\end{align*}
A term with zero true probability contributes zero to the affinity;
this is why the formula remains valid when preference probabilities
are at the boundary of \([0,1]\).

The likelihood product identity gives
\(W_s(Q)=W_{s-1}(Q)
 \sqrt{f_Q(Y_s\mid Z_s)/f_P(Y_s\mid Z_s)}\).
Since \(W_{s-1}(Q)\) is nonnegative and
\(\mathscr H_{s-1}\)-measurable, the tower property and
\eqref{eq:gp_conditional_hellinger_gap} imply
\begin{align*}
\EE_P[W_s(Q)\mid\mathscr H_{s-1}]
&=W_{s-1}(Q)\left(1-
  \EE_P[h^2(P(Z_s),Q(Z_s))\mid\mathscr H_{s-1}]\right)\\
&\le(1-\kappa)W_{s-1}(Q).
\end{align*}
Iterating expectations gives
\[
\EE_P[W_t(Q)]\le(1-\kappa)^t\le e^{-\kappa t}.
\]
Here \(0<\kappa<1\), since \(h_\star\le1\) and \(\eta>0\), and
\(1-\kappa\le e^{-\kappa}\). In particular these expectations are
finite. The sampling policies need not be independent of the past:
the contraction was established conditionally on that past. The
likelihood comparison uses the same realized sampling rule for both
candidates; it does not replace the observed policies by policies
that would have been generated under a different candidate.

Exact maximum likelihood and feasibility of the truth imply
\(\Lambda_t(\widehat P_t)\ge\Lambda_t(P)\).
Hence \(\{\widehat P_t=Q\}\subseteq\{W_t(Q)\ge1\}\), and Markov's
inequality gives
\[
\PP_P(\widehat P_t=Q)
\le\PP_P(W_t(Q)\ge1)
\le\EE_P[W_t(Q)]\le e^{-\kappa t}.
\]
Only candidates with positive regret need to be included. There are
at most \(M-1\) such candidates, since the truth has zero regret under
itself. A union bound therefore yields
\[
\PP_P\!\left(\mathcal G_P^{\mathrm{GP}}(\widehat P_t)>0\right)
\le\sum_{Q:\,\mathcal G_P^{\mathrm{GP}}(Q)>0}
                  \PP_P(\widehat P_t=Q)
\le(M-1)e^{-\kappa t}.
\]
Since \(0\le\mathcal G_P^{\mathrm{GP}}(Q)\le1\), its expectation is
at most the probability of positive regret and is also at most one.
This proves the two one-round bounds.

The initialization contributes at most one. Extending the remaining
nonnegative sum to all positive times and evaluating the geometric
series gives
\begin{align*}
\sum_{t=0}^{T-1}\EE_P[\mathcal G_P^{\mathrm{GP}}(\widehat P_t)]
&\le1+(M-1)\sum_{t=1}^{\infty}e^{-\kappa t}\\
&=1+(M-1)\frac{e^{-\kappa}}{1-e^{-\kappa}}
=1+\frac{M-1}{e^\kappa-1}.
\end{align*}
The constants \(M\), \(h_\star\), and \(\kappa\) do not depend on
the realized truth or horizon. Taking both suprema completes the proof.
\end{proof}

\end{document}